\documentclass[twoside,11pt]{article}

\usepackage{jmlr}
\usepackage{amsmath}
\usepackage{algorithm}
\usepackage{algorithmic}
\usepackage{subfigure}
\usepackage{multirow}

\jmlrheading{1}{2000}{1-48}{4/00}{10/00}{meila00a}{Shen-Yi Zhao and Wu-Jun Li}

\def\a{{\bf a}}

\def\d{{\bf d}}

\def\g{{\bf g}}

\def\u{{\bf u}}
\def\v{{\bf v}}
\def\w{{\bf w}}

\def\z{{\bf z}}

\def\D{{\bf D}}

\def\H{{\bf H}}
\def\I{{\bf I}}

\def\0{{\bf 0}}
\def\1{{\bf 1}}
\def\2{{\bf 2}}
\def\3{{\bf 3}}
\def\4{{\bf 4}}
\def\5{{\bf 5}}
\def\6{{\bf 6}}
\def\7{{\bf 7}}
\def\8{{\bf 8}}
\def\9{{\bf 9}}

\def\NM{{\mathcal N}}

\def\DB{{\mathbb D}}
\def\EB{{\mathbb E}}

\def\PB{{\mathbb P}}

\def\RB{{\mathbb R}}

\begin{document}

\title{Stagewise Enlargement of Batch Size for SGD-based Learning}

\author{\name Shen-Yi Zhao \email zhaosy@lamda.nju.edu.cn \\
		\name Yin-Peng Xie \email xieyp@lamda.nju.edu.cn \\
		\name Wu-Jun Li \email liwujun@nju.edu.cn \\
        \addr Department of Computer Science and Technology\\
        Nanjing University, China}


\maketitle

\begin{abstract}
Existing research shows that the batch size can seriously affect the performance of stochastic gradient descent~(SGD) based learning, including training speed and generalization ability. A larger batch size typically results in less parameter updates. In distributed training, a larger batch size also results in less frequent communication. However, a larger batch size can make a generalization gap more easily. Hence, how to set a proper batch size for SGD has recently attracted much attention. Although some methods about setting batch size have been proposed, the batch size problem has still not been well solved. In this paper, we first provide theory to show that a proper batch size is related to the gap between initialization and optimum of the model parameter. Then based on this theory, we propose a novel method, called \underline{s}tagewise \underline{e}nlargement of \underline{b}atch \underline{s}ize~(\mbox{SEBS}), to set proper batch size for SGD. More specifically, \mbox{SEBS} adopts a multi-stage scheme, and enlarges the batch size geometrically by stage. We theoretically prove that, compared to classical stagewise SGD which decreases learning rate by stage, \mbox{SEBS} can reduce the number of parameter updates without increasing generalization error. SEBS is suitable for \mbox{SGD}, momentum \mbox{SGD} and AdaGrad. Empirical results on real data successfully verify the theories of \mbox{SEBS}. Furthermore, empirical results also show that SEBS can outperform other baselines.

\end{abstract}


\section{Introduction}
Many machine learning models can be formulated as the following empirical risk minimization~(ERM) problem:
\begin{align}\label{equation:erm}
  \min_{\w\in \RB^d} F(\w) = \frac{1}{n}\sum_{i=1}^{n} f(\w;\xi_i),
\end{align}
where $\w$ denotes the model parameter, $\mathcal{I} = \{\xi_1,\xi_2,\ldots,$ $\xi_n\}$ denotes the set of training instances sampled from distribution $\DB$, and $f(\w;\xi_i)$ denotes the loss on the $i$-th training instance.

With the rapid growth of data, stochastic gradient descent~(SGD) and mini-batch SGD~\citep{robbins1951,Bottou:1999:OLS:304710.304720} have become the most popular methods for solving the ERM problem in~(\ref{equation:erm}), and many variants of SGD have been proposed. Among these algorithms, the classical and most widely used one is the stagewise SGD which has been adopted in~\citep{DBLP:conf/nips/KrizhevskySH12,DBLP:conf/cvpr/HeZRS16}. Stagewise SGD is based on a multi-stage learning scheme. At the $s$-th stage, it runs the following iterations:
\begin{equation}\label{equation:sgd}
	\w_{m+1} = \w_{m} - \eta_s(\frac{1}{b}\sum_{\xi\in \mathcal{B}_m}\nabla f(\w_{m};\xi)),
\end{equation}
where $\w_1 = \tilde{\w}_s$ is the initialization, $m=1,2,\ldots,M_s$, $\mathcal{B}_m \subset\mathcal{I}$ is a mini-batch of instances randomly sampled from $\mathcal{I}$ with a \emph{batch size} $|\mathcal{B}_m| = b$, $\eta_s$ is the learning rate which is a constant at each stage and decreases geometrically by stage. After the $s$-th stage is completed, the algorithm randomly picks a parameter from $\{\w_{m}\}$ or the last one $\w_{M_s+1}$ as the initialization of the next stage. For stagewise SGD with $S$ stages, the \emph{computation complexity}~(total number of gradient computation) is $\sum_s^S M_s b$ and the \emph{iteration complexity}~(total number of parameter updates) is $\sum_s^S M_s$. Recently, some work~\citep{DBLP:conf/nips/NIPS2019_8529} theoretically proves that the stagewise SGD is better than the original SGD which adopts the polynomially decreased learning rate under the weakly quasi-convex and Polyak-Lojasiewicz~(PL) condition. Classical stagewise SGD methods~\citep{DBLP:conf/nips/KrizhevskySH12,DBLP:conf/cvpr/HeZRS16} mainly focus on how to set the learning rate for a given constant batch size which is typically not too large.

From~(\ref{equation:sgd}), we can find that given a fixed computation complexity, a larger batch size will result in less parameter updates. In distributed training, each parameter update typically needs one time of communication, and hence a larger batch size will result in less frequent communication. Furthermore, a larger batch size can typically better utilize the computing power of current multi-core systems like \mbox{GPU} to reduce computation time, as long as the mini-batch does not exceed the memory or computing limit of the system. Figure~\ref{figure:training speed} gives an example to show that enlarging batch size can reduce computation time. Hence, we need to choose a larger batch size for SGD to reduce computation time if we do not take generalization error into consideration. However, a larger batch size can make a generalization gap more easily~\citep{DBLP:conf/iclr/KeskarMNST17,LeCun2012}. Some work~\citep{DBLP:conf/nips/HofferHS17} points out that we need to train longer~(with higher computation complexity) for larger batch training to achieve a similar generalization error as that of smaller batch training. This is contrary to the original intention of large batch training. Hence, how to set a proper batch size for SGD has become an interesting but challenging topic.

\begin{figure}[t]
\centering
\includegraphics[width = 7cm]{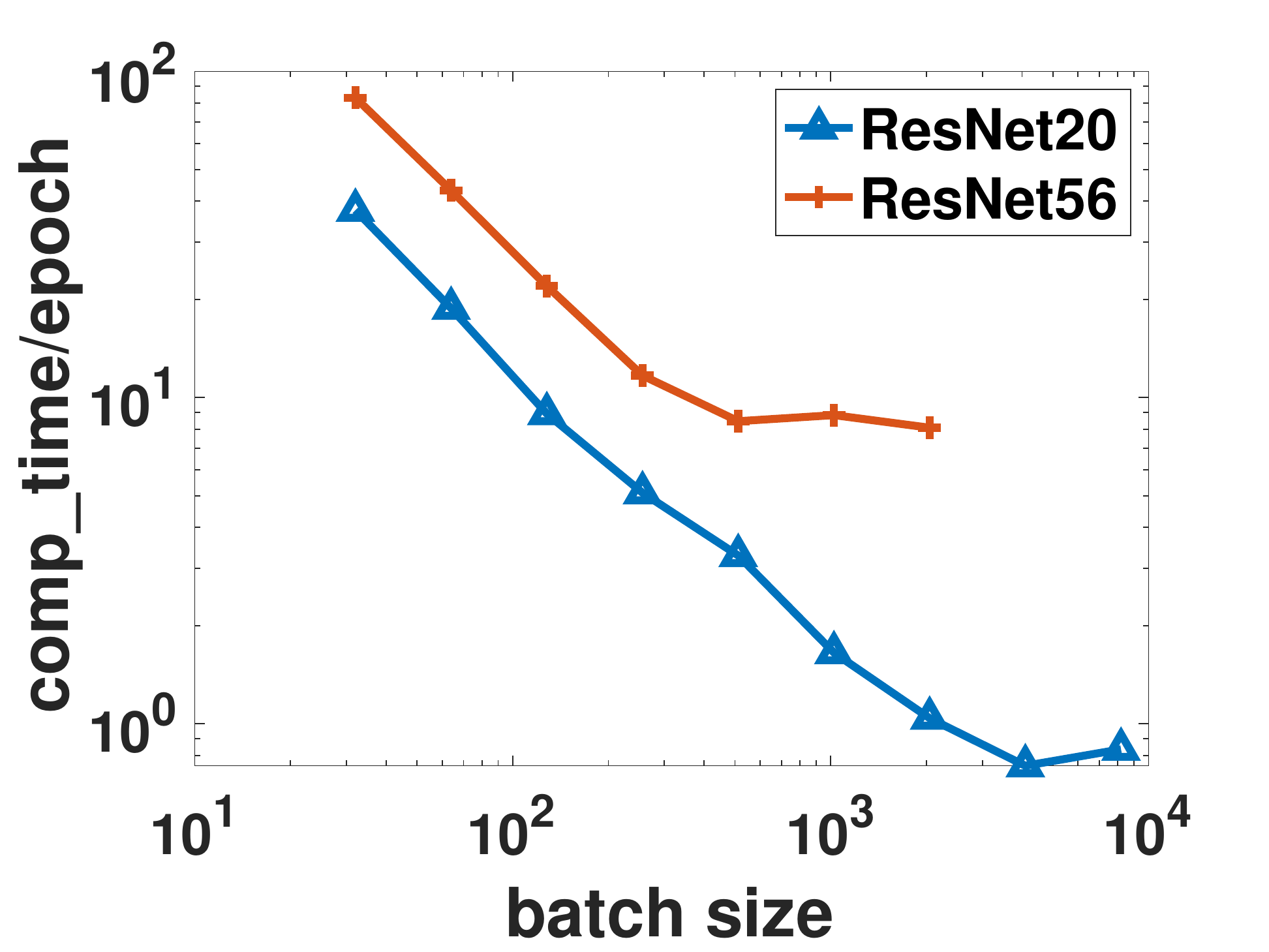}
\caption{Using an NVIDIA V100 GPU to train models on \mbox{CIFAR10}. The y-axis denotes computation time per epoch. In ResNet20, volatile gpu-util achieves $100\%$ when $b = 4096$. In ResNet56, volatile gpu-util achieves $100\%$ when $b = 512$.}\label{figure:training speed}
\vspace{-0.5cm}
\end{figure}

There have appeared some works proposing heuristic methods for large batch training~\citep{DBLP:journals/corr/GoyalDGNWKTJH17,DBLP:journals/corr/abs-1708-03888,DBLP:journals/corr/abs-1812-06162}.  Compared to classical stagewise SGD methods with a small constant batch size and stagewisely decreased learning rate, these large batch methods need more tricks, which should be carefully tuned on different models and data sets. Furthermore, theoretical guarantee about the iteration complexity and generalization error of these methods is missing. In addition, in our experiments we find that these methods might increase generalization error if a large batch size is adopted from the initialization.

There have also appeared some other methods proposing to dynamically set the batch size. ~\citep{DBLP:journals/siamsc/FriedlanderS12,DBLP:journals/mp/ByrdCNW12,DBLP:conf/aistats/DeYJG17,DBLP:conf/aistats/YinPLPRB18} relate the batch size with the noise of stochastic gradients. These methods need to determine the batch size in each iteration, which will bring much extra cost. ~\citep{DBLP:conf/iclr/SmithKYL18} increases the batch size by relating SGD with a stochastic differential equation. However, the theoretical guarantee about the iteration complexity and generalization error is missing. Furthermore, some work~\citep{DBLP:conf/icml/YuJ19} uses the stagewise training strategy. At each stage, the batch size starts from a small constant and is geometrically increased by iteration. However, the scaling ratio for the batch size cannot be large for convergence guarantee. Furthermore, in our experiments we also find that it might increase generalization error.

In this paper, we propose a novel method, called \underline{s}tagewise \underline{e}nlargement of \underline{b}atch \underline{s}ize~(\mbox{SEBS}), to set proper batch size for SGD. The main contributions of this paper are outlined as follows:
\begin{itemize}
	\item We first provide theory\footnote{Due to space limitation, we only present the Lemmas and Theorems in the main text, and the detailed proof can be found in the supplementary material.} to show that a proper batch size is related to the gap between initialization and optimum of the model parameter. Then based on this theory, we propose \mbox{SEBS} which adopts a multi-stage scheme and enlarges the batch size geometrically by stage.
\item We theoretically prove that decreasing learning rate and enlarging batch size have the same effect on the performance of stagewise SGD.
\item We theoretically prove that, compared to classical stagewise SGD which decreases learning rate by stage, \mbox{SEBS} can reduce the number of parameter updates~(iteration complexity) without increasing generalization error when the total number of gradient computation~(computation complexity) is fixed.
\item Besides \mbox{SGD}, SEBS is also suitable for momentum \mbox{SGD} and adaptive gradient descent~(\mbox{AdaGrad})~\citep{DBLP:conf/colt/DuchiHS10}. We also provide theoretical results about the number of parameter updates~(iteration complexity) for momentum \mbox{SGD} and AdaGrad. To the best of our knowledge, this is the first work that analyzes the effect of batch size on the convergence of AdaGrad~\footnote{In this paper, AdaGrad refers to the coordinate form adaptive gradient descent~\citep{DBLP:conf/colt/DuchiHS10}.}.
\item Empirical results on real data successfully verify the theories of \mbox{SEBS}. Furthermore, empirical results also show that SEBS can outperform other baselines.
\end{itemize}

\section{Preliminaries}
First, we give the following notations. $\|\cdot\|$ denotes the $L_2$ norm. $\|\cdot\|_\infty$ denotes the $L_{\infty}$ norm. $\w^*$ denotes the optimal solution~(optimum) of~(\ref{equation:erm}). $\nabla f_{\mathcal{B}}(\w) \triangleq \frac{1}{|\mathcal{B}|}\sum_{i\in \mathcal{B}}\nabla f(\w;\xi_i)$ denotes the stochastic gradient of the mini-batch $\mathcal{B}$. $\forall \a\in \RB^d$, we use $a^{(j)}$ to denote the $j$-th element of $\a$.

We also make the following assumptions.
\begin{assumption}\label{assumption:bounded var}
The variance of stochastic gradient is bounded: $\forall \w$, $\EB_{i\sim[n]}\|\nabla f(\w;\xi_i) - \nabla F(\w)\| \leq \sigma^2$.
\end{assumption}

\begin{assumption}\label{assumption:smooth}
  $f(\w;\xi)$ is $L$-smooth~($L>0$): $\forall \w, \w', $ $\xi\sim \DB$, $\|\nabla f(\w;\xi) - \nabla f(\w';\xi)\| \leq L\|\w - \w'\|$.
\end{assumption}

\begin{assumption}\label{assumption:quasi convex}
  $F(\w)$ is $\alpha$-weakly quasi-convex~($\alpha>0$):
  \begin{align*}
    \nabla F(\w)^T(\w - \w^*) \geq \alpha(F(\w) - F(\w^*)), \forall \w.
  \end{align*}
\end{assumption}

\begin{assumption}\label{assumption:pl condition}
  $F(\w)$ satisfies $\mu$-Polyak Lojasiewicz~($\mu$-PL, $\mu>0$) condition:
  \begin{align*}
    \|\nabla F(\w)\|^2 \geq 2\mu(F(\w) - F(\w^*)), \forall \w.
  \end{align*}
\end{assumption}
Recently, both weak quasi-convexity and PL condition have been observed for many machine learning models, including deep neural networks~\citep{DBLP:conf/icml/CharlesP18,DBLP:conf/nips/NIPS2019_8529}. The $\mu$-PL condition also implies a quadratic growth~\citep{DBLP:conf/pkdd/KarimiNS16}, i.e., $F(\w) - F(\w^*) \geq \mu\|\w - \w^*\|^2/2$. Another inequality~\citep{DBLP:books/sp/Nesterov04} used in this paper is $F(\w) - F(\w^*) \geq \|\nabla F(\w)\|^2/(2L)$. Please note that these two inequalities do not need the convex assumption. We call $\kappa = L/\mu$ the conditional number of $F(\w)$ under PL condition.

\section{SEBS}
In this section, we present the details of SEBS for SGD, including the theory about the relationship between batch size and model initialization, SEBS algorithm, theoretical analysis about the training error and generalization error.

\subsection{Relationship between Batch Size and Model Initialization}\label{sec:relation}
We start from the vanilla SGD with a constant batch size and learning rate, which can be written as follows:
\begin{align}\label{equation:update of vanilla sgd}
	\w_{m+1} = \w_m - \eta \nabla f_{\mathcal{B}_m}(\w_m),
\end{align}
where $m=1,2,\ldots,M$, and $|\mathcal{B}_m| = b$. The computation complexity~(total number of gradient computation) is $C = Mb$. Let $\hat{\w}$ denote a value randomly sampled from $\{\w_2,\ldots,\w_{M+1}\}$.

We aim to find how large the batch size can be without loss of performance. First, we can obtain the following property about (\ref{equation:update of vanilla sgd}):
\begin{lemma}\label{lemma:vanilla sgd}
By setting $\eta \leq \alpha/(2L)$, we have
	\begin{align}\label{equation:vanilla sgd}
		\EB[F(\hat{\w}) - F(\w^*)] \leq \frac{\|\w_1 - \w^*\|^2}{\alpha M\eta} + \frac{\eta\sigma^2}{\alpha b}.
	\end{align}
\end{lemma}
\begin{remark}
	Another common upper bound for $\EB[F(\hat{\w}) - F(\w^*)]$ is from~\citep{DBLP:conf/icml/Zinkevich03}:
	\begin{align*}
		\EB[F(\hat{\w}) - F(\w^*)] \leq \frac{\|\w_1 - \w^*\|^2}{2M\eta} + \frac{\eta G^2}{2},
	\end{align*}
	which uses the bounded gradient assumption $\EB_{\xi}\|\nabla f(\w;\xi)\|^2\leq G^2, \forall \w$. Comparing to Assumption~\ref{assumption:bounded var}, we can see that the bounded gradient assumption in~\citep{DBLP:conf/icml/Zinkevich03} omits the effect of batch size.
\end{remark}

Based on~(\ref{equation:vanilla sgd}), we can get a learning rate $\mathcal{O}(1/\sqrt{M})$ which minimizes the right term of~(\ref{equation:vanilla sgd}). In fact, (\ref{equation:vanilla sgd}) also implies a proper batch size. Using $C = Mb$, we rewrite the right term of~(\ref{equation:vanilla sgd}) as follows:
\begin{align*}
	\psi(\eta,b) = \frac{b\|\w_1 - \w^*\|^2}{\alpha C\eta} + \frac{\eta\sigma^2}{\alpha b}.
\end{align*}
Then, we have:  $\forall \eta>0, b>0$,
\begin{align*}
	\psi(\eta,b) \geq 2\|\w_1 - \w^*\|\sigma/(\alpha\sqrt{C}).
\end{align*}

To make $\psi(\eta,b)$ get the minimum, the corresponding batch size $b^*$ and learning rate $\eta^*$ should satisfy:
\begin{align}\label{equation:basic relation}
	\eta^* = \frac{\|\w_1 - \w^*\|b^*}{\sigma\sqrt{C}} \leq \frac{\alpha}{2L},
\end{align}
where $\eta^*\leq \alpha/(2L)$ is from Lemma~\ref{lemma:vanilla sgd}.

From~(\ref{equation:basic relation}), we can find that given a fixed computation complexity $C$, a proper batch size is related to the gap between the initialization and optimum of the model parameter. More specifically, the smaller the gap between the initialization and optimum of the model parameter is, the larger the batch size can be.

The theory of this subsection provides theoretical foundation for designing the SEBS algorithm in the following subsection.

\subsection{SEBS Algorithm}
In classical stagewise SGD~\citep{DBLP:conf/nips/KrizhevskySH12,DBLP:conf/cvpr/HeZRS16}, we can see that at each stage it actually runs the vanilla SGD with a constant batch size and learning rate. After each stage, it decreases the learning rate geometrically. In~\citep{DBLP:conf/nips/NIPS2019_8529}, both theoretical and empirical results show that after each stage there is a geometric decrease in the training loss. This means that the gap between the current model parameter and the optimal solution~(optimum) $\w^*$ is smaller than that of previous stages. Based on the theory about the relationship between the batch size and model initialization from Section~\ref{sec:relation}, we can actually enlarge the batch size in the next stage. Inspired by this, we propose our algorithm called stagewise enlargement of batch size~(SEBS) for SGD-based learning.

SEBS adopts a multi-stage scheme, and enlarges the batch size geometrically by stage. The detail of \mbox{SEBS} is presented in Algorithm~\ref{alg:mb_sgd}. We can find that SEBS divides the whole learning procedure into $S$ stages. At the $s$-th stage, SEBS runs the penalty SGD in Algorithm~\ref{alg:psgd}, denoted as $\mbox{\textit{pSGD}}(f, \mathcal{I}, \gamma, \tilde{\w}_s, \eta, b_s, C_s)$. Here, $f$ denotes the loss function in (\ref{equation:erm}), $\mathcal{I}$ denotes the training set, $\gamma$ is the coefficient of a quadratic penalty, $\tilde{\w}_s$ is the initialization of the model parameter at the $s$-th stage, $b_s$ is the batch size at the $s$-th stage, $\eta$ is a constant learning rate, and $C_s$ is the computation complexity at the $s$-th stage. The output of $\mbox{\textit{pSGD}}$, denoted as $\tilde{\w}_{s+1}$, will be used as the model parameter initialization for the next stage.

The penalty SGD is a variant of vanilla SGD. Compared to the vanilla SGD, there is an additional quadratic penalty $r(\w) = \frac{1}{2\gamma}\|\w - \tilde{\w}\|^2$ in penalty SGD. If $\gamma = \infty$, penalty SGD degenerates to the vanilla SGD. The quadratic penalty has been widely used in many recent variants of SGD~\citep{DBLP:conf/nips/Allen-Zhu18,DBLP:conf/icml/YuJ19,DBLP:conf/iclr/ChenYYZCY19,DBLP:conf/icml/ChenXHY19,DBLP:conf/nips/NIPS2019_8529}. Although it may slow down the convergence rate, it can improve the generalization ability.

\begin{algorithm}[t]
\caption{SEBS}\label{alg:mb_sgd}
\begin{algorithmic}
\STATE Initialization: $\tilde{\w}_1, \eta, b_1, C_1, \gamma>0, \rho>1$.
\FOR{$s=1,2,\ldots,S$}
\STATE $\tilde{\w}_{s+1} = \mbox{\textit{pSGD}}(f, \mathcal{I}, \gamma, \tilde{\w}_s, \eta, b_s, C_s)$;
\STATE $b_{s+1} = \rho b_s, ~C_{s+1}=\rho C_s$;
\ENDFOR
\STATE Return $\tilde{\w}_{S+1}$.
\end{algorithmic}
\end{algorithm}
\begin{algorithm}[t]
\caption{$\mbox{\textit{pSGD}}(f,\mathcal{I}, \gamma, \tilde{\w}, \eta, b, C)$}\label{alg:psgd}
\begin{algorithmic}
\STATE Initialization: $\w_1 = \tilde{\w}, M = C/b$;
\STATE Let $r(\w) = \frac{1}{2\gamma}\|\w - \tilde{\w}\|^2$;
\FOR{$m=1,2,\ldots,M$}
\STATE Randomly select $\mathcal{B}_m \subseteq [n]$ and $|\mathcal{B}_m| = b;$
\STATE Calculate gradient $\g_m = \nabla f_{\mathcal{B}_m}(\w_m)$;
\STATE $\w_{m+1} = \underset{\w}{\arg\min}~\g_m^T\w + \frac{1}{2\eta}\|\w - \w_m\|^2 + r(\w)$.
\ENDFOR
\STATE Return $\w_\tau$ which is randomly sampled from $\{\w_m\}_{m=2}^{M+1}$.
\end{algorithmic}
\end{algorithm}

\subsection{Theoretical Analysis about Training Error}
First, we have the following one-stage training error for SEBS:
\begin{lemma}\label{lemma:one stage error for pSGD}
	(One-stage training error for SEBS) \\ Let $\{\w_m\}$ be the sequence produced by $\mbox{\textit{pSGD}}(f,\mathcal{I},$ $\gamma, \tilde{\w}, \eta, b, C)$, where $\eta \leq \alpha/(2L)$. Then we have:
	\begin{align}
	    & \EB[F(\w_\tau) - F(\w^*)] \nonumber \\
	    \leq & (\frac{1}{\alpha M\eta} + \frac{1}{\alpha\gamma})\|\tilde{\w} - \w_*\|^2 + \frac{\sigma^2\eta}{\alpha b},\label{equation:lemma T}
	\end{align}
	where $\w_\tau$ is the output of $\mbox{\textit{pSGD}}$ and $M = C/b$.
\end{lemma}
We can find that the one-stage training error for SEBS is similar to that in~(\ref{equation:vanilla sgd}). Hence, we can set the batch size of each stage in SEBS according to the gap $\|\tilde{\w} - \w_*\|$. Particularly, we can get the following convergence result:
\begin{theorem}\label{theorem:mb_sgd}
	Let $F(\tilde{\w}_1) - F(\w^*) \leq \epsilon_1$ and $\{\tilde{\w}_s\}$ be the sequence produced by
	$$\tilde{\w}_{s+1} = \mbox{\textit{pSGD}}(f, \mathcal{I}, \gamma, \tilde{\w}_s, \eta_s, b_s, C_s),$$ where $C_s = \theta/\epsilon_s$, and
	\begin{align}\label{equation:key relation}
		\eta_s = \frac{\sqrt{2}b_s\epsilon_s}{\sigma\sqrt{\mu\theta}} \leq \frac{\alpha}{2L}.
	\end{align}
	Then we obtain $\EB[F(\tilde{\w}_{s}) - F(\w^*)] \leq \epsilon_s, \forall s\geq 1$. If $S =\log_\rho(\epsilon_1/\epsilon)$, then $\EB[F(\tilde{\w}_{S+1}) - F(\w^*)] \leq \epsilon$. Here, $\frac{1}{\gamma} \leq \frac{\alpha\mu}{4\rho}$, $\theta= 32\sigma^2\rho^2/(\alpha^2\mu)$ and $\epsilon_{s+1} = \epsilon_s/\rho, s\geq 1, \rho>1$.
\end{theorem}

In SEBS, if we set $\eta_s = \eta = \alpha/(2L)$ which is a constant, and set the batch size as
\begin{align}\label{equation:batch size}
	b_s = \frac{\sigma\eta\sqrt{\mu\theta}}{\sqrt{2}\epsilon_s} = \frac{\alpha\sigma\sqrt{\mu\theta}}{2\sqrt{2}L\epsilon_s} = \mathcal{O}(\frac{1}{F(\tilde{\w}_s) - F(\w^*)}),
\end{align}
which means $b_{s+1} = \rho b_s$, according to Theorem~\ref{theorem:mb_sgd}, we can obtain the computation complexity of SEBS:
\begin{align*}
	\sum_{s=1}^S C_s = \sum_{s=1}^S \frac{\theta}{\epsilon_s} \leq \mathcal{O}(\frac{\sigma^2}{\alpha^2\mu\epsilon}).
\end{align*}
This result is consistent with that in~\citep{DBLP:conf/nips/NIPS2019_8529} which sets $\rho = 2, b_s = 1, \eta_{s+1} = \eta_s/2$. Hence, by setting the batch size $b_s$ according to~(\ref{equation:batch size}), SEBS achieves the same performance as classical stagewise SGD on computation complexity. Please note that when the loss function $F(\w)$ is strongly convex, which means $\alpha\geq 1$, the proved computation complexity above is optimal~\citep{DBLP:conf/icml/RakhlinSS12}.

The iteration complexity of SEBS is as follows:
\begin{align*}
	\sum_{s=1}^S \frac{C_s}{b_s} = \sum_{s=1}^S \mathcal{O}(\frac{L\sqrt{\theta}}{\alpha\sigma\sqrt{\mu}}) = \mathcal{O}(\frac{L}{\alpha^2\mu}\log(\frac{1}{\epsilon})).
\end{align*}

Then we can get the following conclusions:
\begin{itemize}
	\item Compared to classical stagewise SGD which decreases learning rate by stage and adopts a constant batch size, SEBS reduces the iteration complexity from $\mathcal{O}(\frac{G^2}{\alpha^2\mu\epsilon})$ to $\mathcal{O}(\frac{1}{\alpha^2\mu}\log(\frac{1}{\epsilon}))$, where $G$ is the upper bound for $\|\nabla f(\w;\xi)\|$;
	\item We can also observe that the iteration complexity of SEBS is independent of the variance $\sigma^2$, and hence is independent of the dimension $d$;
	\item According to (\ref{equation:key relation}) in Theorem~\ref{theorem:mb_sgd}, in order to get the convergence result, we need to keep the relation between the batch size and learning rate in each stage as follows:
	      \begin{align*}
	      	  \frac{\eta_s}{b_s} = \mathcal{O}(\epsilon_s).
	      \end{align*}
	     This relation implies that the following two strategies for adjusting batch size and learning rate:
	     \begin{align}\label{strategy:a}
	     	\mbox{\textit{constant batch size $\&$ decrease learning rate}} \tag{a}
	     \end{align}
	     and
	     \begin{align}\label{strategy:b}
	     	\mbox{\textit{constant learning rate $\&$ enlarge batch size}} \tag{b}
	     \end{align}
	     are equivalent in terms of training error. Both of them will not affect the computation complexity. Please note that strategy~(a) has been widely adopted in classical stagewise training methods, and strategy~(b) is proposed in SEBS.
\end{itemize}

\subsection{Theoretical Analysis about Generalization Error}
In this section, we will analyze the generalization error of SEBS. The main tool we used for the generalization error is the uniform stability~\citep{DBLP:conf/icml/HardtRS16}, which is defined as follows:
\begin{definition}
	A randomized algorithm $\mathcal{A}$ is $\epsilon$-uniformly stable if for all data sets $\mathcal{I}_1$, $\mathcal{I}_2$ such that $\mathcal{I}_1$ and $\mathcal{I}_2$ differ in at most one instance, we have
	\begin{align*}
		\epsilon_{stab} \triangleq \sup_{\xi} \EB_{\mathcal{A}}[f(\tilde{\w}_1;\xi) - f(\tilde{\w}_2;\xi)] \leq \epsilon,
	\end{align*}
	where $\tilde{\w}_i$ is the output of $\mathcal{A}$ on data set $\mathcal{I}_i$, $i=1,2$.
\end{definition}
It has been proved~\citep{DBLP:conf/icml/HardtRS16} that if $\mathcal{A}$ is $\epsilon$-uniformly stable, then
\begin{align*}
	|\EB_{\mathcal{I},\mathcal{A}}[F(\tilde{\w}) - \EB_{\xi\sim \DB}[f(\tilde{\w};\xi)]]| \leq \epsilon,
\end{align*}
where $\tilde{\w}$ is the output of $\mathcal{A}$ on data set $\mathcal{I}$. Hence, in the following content, we consider the two data sets $\mathcal{I}_1 = \{\xi_1,\ldots,\xi_{i_0},\ldots,\xi_n\}$ and $\mathcal{I}_2 = \{\xi_1,\ldots,\xi_{i_0}',\ldots,\xi_n\}$ differing in only a single instance which is indexed by $i_0$. Let $\tilde{\w}_i$ be the output of SEBS on data set $\mathcal{I}_i$, $\{\w_{i,m}\}$ be the sequences produced by SEBS at the last stage, $\{\mathcal{B}_{i,m}\}$ be the corresponding randomly selected mini-batch of instances, $i=1,2$. We omit the subscript $s$ and use $\eta, b, C$ to denote the learning rate, batch size and computation complexity in the last stage.
We also define $\delta_m = \|\w_{1,m} - \w_{2,m}\|$. Following~\citep{DBLP:conf/icml/HardtRS16}, we assume that $\|\nabla f(\w;\xi)\| \leq G, 0\leq f(\w;\xi) \leq 1, \forall \w, \xi\sim \DB$. Then we have the following property about $\delta_m$.
\begin{lemma}\label{lemma:delta_m}
	For one specific $m$, if $\mathcal{B}_{1,m} = \mathcal{B}_{2,m} \triangleq \mathcal{B}_{m}$, then we get
	\begin{align*}
		\delta_{m+1} \leq \frac{\eta}{\gamma+\eta}\delta_1 + \frac{\gamma(1+L\eta)}{\gamma + \eta}\delta_m.
	\end{align*}
	If $\mathcal{B}_{1,m} \neq \mathcal{B}_{2,m}$, then we get
	\begin{align*}
		\delta_{m+1} \leq \frac{\eta}{\gamma + \eta}\delta_1 + \frac{b\gamma + (b-1)L\gamma\eta}{b(\gamma + \eta)}\delta_m + \frac{2\gamma\eta G}{b(\gamma + \eta)}.
	\end{align*}
\end{lemma}
Using the recursive relation of $\delta_m$, we get the following uniform stability of SEBS:
\begin{theorem}\label{theorem:stability}
	With the $\eta_s, b_s, C_s$ defined in Theorem~\ref{theorem:mb_sgd}, we obtain
	\begin{align*}
		\epsilon_{stab} \leq \frac{C}{n} + \frac{(1 + 1/q)}{n}(\frac{4\gamma G^2}{(\gamma+\eta)\mu\alpha})^{\frac{1}{1+q}}C^{\frac{q}{q+1}},
	\end{align*}
	where $q = \frac{2L}{\mu\alpha}$.
\end{theorem}
According to Theorem~\ref{theorem:stability}, we obtain the following two conclusions:
\begin{itemize}
	\item This uniform stability is consistent with~\citep{DBLP:conf/nips/NIPS2019_8529}. The stability error only depends on the computation complexity and has nothing to do with the batch size of each stage.
	\item Compared to the classical non-penalty SGD in~(\ref{equation:sgd}) which actually corresponds to the penalty SGD with $\gamma = \infty$ and $\gamma/(\eta + \gamma) \approx 1$, penalty SGD with a finite $\gamma$ can improve the stability, and hence improve the generalization error.
\end{itemize}
Since $C = \mathcal{O}(1/(\alpha^2\mu\epsilon))$ and $(\frac{4\gamma G^2}{(\gamma+\eta)\mu\alpha})^{\frac{1}{1+q}} \leq e^{2G^2/(2L)}$,
by setting $\epsilon = \mathcal{O}(1/\sqrt{n})$, we obtain a generalization error $\epsilon + \epsilon_{stab} = \mathcal{O}(1/\sqrt{n})$ for SEBS.

\section{SEBS for Momentum SGD and AdaGrad}
Momentum SGD~(mSGD)~\citep{Polyak,DBLP:journals/siamjo/Tseng98,DBLP:journals/siamjo/GhadimiL13,DBLP:journals/mp/GhadimiL16} and adaptive gradient descent~(AdaGrad)~\citep{DBLP:conf/colt/McMahanS10,DBLP:conf/colt/DuchiHS10} have been two of the most important and popular variants of SGD. In the following content, we will show that SEBS is also suitable for momentum SGD and \mbox{AdaGrad}. To the best of our knowledge, existing research on AdaGrad only analyzes the convergence property with $b=1$. This is the first work that analyzes the effect of batch size on the convergence of AdaGrad.

\subsection{SEBS for Momentum SGD}
Here, we propose to adapt SEBS for momentum SGD. The resulting algorithm is called mSEBS, which is presented in Algorithm~\ref{alg:mb_msgd}. mSEBS divides the whole learning procedure into $S$ stages. At each stage, mSEBS runs the Polyak's momentum SGD~\citep{Polyak} which is presented in Algorithm~\ref{alg:msgd}. Please note that mSEBS will reset the momentum to zero after each stage for the convenience of convergence proof. This is different from some mSGD implementations like that on PyTorch which does not reset the momentum to zero. In our experiments, we find that this difference does not have significant influence.

Similar to SEBS for SGD, mSEBS can also achieve $\mathcal{O}(\frac{\sigma^2}{\alpha^2\mu\epsilon})$ computation complexity and $\mathcal{O}(\frac{L}{\alpha^2\mu}\log(\frac{1}{\epsilon}))$ iteration complexity which is independent of $\sigma$ and $d$. Due to space limitation, we move the related theorems to the supplementary material.

\begin{algorithm}[t]
\caption{SEBS for Momentum SGD~(mSEBS)}\label{alg:mb_msgd}
\begin{algorithmic}
\STATE Initialization: $\tilde{\w}_1, \eta, b_1, C_1, \beta\in[0,1), \rho>1$.
\FOR{$s=1,2,\ldots,S$}
\STATE $\tilde{\w}_{s+1} = \mbox{\textit{mSGD}}(f, \mathcal{I}, \beta, \tilde{\w}_s, \eta, b_s, C_s)$;
\STATE $b_{s+1} = \rho b_s, ~C_{s+1}=\rho C_s$;
\ENDFOR
\STATE Return $\tilde{\w}_{S+1}$.
\end{algorithmic}
\end{algorithm}

\begin{algorithm}[t]
\caption{$\mbox{\textit{mSGD}}(f,\mathcal{I}, \beta, \tilde{\w}, \eta, b, C)$}\label{alg:msgd}
\begin{algorithmic}
\STATE Initialization: $\u_1 = \0, \w_1 = \tilde{\w}, M = C/b$;
\FOR{$m=1,2,\ldots,M$}
\STATE Randomly select $\mathcal{B}_m \subseteq [n]$ and $|\mathcal{B}_m| = b;$
\STATE Calculate gradient $\g_m = \nabla f_{\mathcal{B}_m}(\w_m)$;
\STATE $\u_{m+1} = \beta\u_m - \eta\g_m$;
\STATE $\w_{m+1} = \w_m + \u_{m+1}$;
\ENDFOR
\STATE Return $\w_\tau$ which is randomly sampled from $\{\w_m\}_{m=2}^{M+1}$.
\end{algorithmic}
\end{algorithm}

\subsection{SEBS for AdaGrad}
Here, we propose to adapt SEBS for AdaGrad. The resulting algorithm is called AdaSEBS, which is presented in Algorithm~\ref{alg:mb_adagrad}. AdaSEBS also divides the whole learning procedure into $S$ stages. At the $s$-th stage, AdaSEBS runs $\mbox{\textit{AdaGrad}}(f,\mathcal{I}, \delta, \nu, \tilde{\w}, \eta, b, C)$ which is presented in Algorithm~\ref{alg:adagrad}. In particular, $\mbox{\textit{AdaGrad}}$ runs the following iterations:
\begin{align}\label{equation:adagrad}
	\w_{m+1} = \arg\min_{\w} \w^T(\sum_{i=1}^m\g_i) + \frac{1}{\eta}\psi_m(\w),
\end{align}
where $\psi_m(\w) = \frac{1}{2}(\w - \tilde{\w})\H_m(\w - \tilde{\w})$. $\H_m$ is a diagonal matrix, in which the diagonal element $h_m^{(j)}$ is defined as $h_m^{(j)} = (\delta^2 + \sum_{i=1}^m (g_{i}^{(j)})^2)^{\nu}$,  where $\nu> 0, \delta \geq \|\nabla f(\w;\xi)\|_{\infty}, \forall \w, \xi$. In existing research, $\nu$ is typically set to $0.5$ for convex loss functions~\citep{DBLP:conf/colt/McMahanS10,DBLP:conf/colt/DuchiHS10} and is typically set to $1$ for strongly convex loss functions~\citep{DBLP:conf/colt/DuchiHS10,DBLP:conf/icml/MukkamalaH17}.

\begin{algorithm}[t]
\caption{SEBS for AdaGrad~(AdaSEBS)}\label{alg:mb_adagrad}
\begin{algorithmic}
\STATE Initialization: $\tilde{\w}_1, \eta, b_1, C_1, \delta > 0, \rho>1$.
\FOR{$s=1,2,\ldots,S$}
\STATE $\tilde{\w}_{s+1} = \mbox{AdaGrad}(f,\mathcal{I}, \delta, 1, \tilde{\w}_s, \eta, b_s, C_s)$.
\STATE $b_{s+1} = \rho b_s, ~C_{s+1}=\rho C_s$;
\ENDFOR
\STATE Return $\tilde{\w}_{S+1}$.
\end{algorithmic}
\end{algorithm}

\begin{algorithm}[t]
\caption{$\mbox{\textit{AdaGrad}}(f,\mathcal{I}, \delta, \nu, \tilde{\w}, \eta, b, C)$}\label{alg:adagrad}
\begin{algorithmic}
\STATE Initialization: $\w_1 = \tilde{\w}, M = C/b, \delta>0, \nu> 0$;
\FOR{$m=1,2,\ldots,M$}
\STATE Randomly select $\mathcal{B}_m \subseteq [n]$ and $|\mathcal{B}_m| = b;$
\STATE Calculate gradient $\g_m = \nabla f_{\mathcal{B}_m}(\w_m)$;
\STATE $h_{m}^{(j)} = (\delta^2 + \sum_{i=1}^m (g_{i}^{(j)})^2)^{\nu},  j=1,2,\ldots d$;
\STATE Let $\psi_m(\w) = \frac{1}{2}(\w - \w_1)\H_m(\w - \w_1)$, where $\H_m = diag(h_{m}^{(1)},\ldots,h_{m}^{(d)})$;
\STATE $\w_{m+1} = \arg\min_{\w} \w^T(\sum_{i=1}^m\g_i) + \frac{1}{\eta}\psi_m(\w)$;
\ENDFOR
\STATE Return $\w_\tau$ which is randomly sampled from $\{\w_m\}_{m=2}^{M+1}$.
\end{algorithmic}
\end{algorithm}

Let $\{\w_m\}$ be the sequence produced by AdaGrad in Algorithm~\ref{alg:adagrad}. According to \citep{DBLP:conf/colt/DuchiHS10}, we obtain
\begin{align*}
	&\frac{\alpha}{M}\sum_{m=1}^M \EB[F(\w_m) - F(\w^*)] \\
	\leq & \frac{1}{M\eta}\EB\psi_M(\w^*) + \frac{\eta}{2M}\sum_{m=1}^M\EB\|\g_m\|^2_{\psi_{m-1}^*},
\end{align*}
where $\|\g_m\|^2_{\psi_{m-1}^*} \triangleq \sum_{j=1}^d (g_m^{(j)})^2/(\delta^2 + \sum_{i=1}^{m-1} (g_i^{(j)})^2)^{\nu}$, and $M = C/b$. If we take $\nu = 0.5$, then $\sum_{m=1}^M\|\g_m\|^2_{\psi_{m-1}^*} \leq 2\sum_{j=1}^d\sqrt{\sum_{m=1}^M (g_m^{(j)})^2}$. When the gradient $g_m^{(j)}$ is relatively small, e.g., $|g_m^{(j)}| \ll 1/M$, the square root operation will make the upper bound of $\sum_{m=1}^M\|\g_m\|^2_{\psi_{m-1}^*}$ bad. Hence in this work, although $f(\w;\xi)$ is not necessarily strongly convex, we still set $\nu = 1$ and get the following one-stage training error for AdaSEBS:

\begin{lemma}\label{lemma:adagrad}
	(One-stage training error for AdaSEBS) \\ Let $\{\w_m\}$ be the sequence produced by $\mbox{AdaGrad}(f,\mathcal{I}, $ $\delta, 1, \tilde{\w},\eta, b, C)$. Then we have
	\begin{align*}
		& (\alpha - \frac{4L\eta}{\delta^2}) \EB[F(\w_\tau) - F(\w^*)] \\
		\leq & \frac{\delta^2}{2M\eta}\|\tilde{\w} - \w^*\|^2 + \frac{2\sigma^2\eta}{b\delta^2},
	\end{align*}
where $\w_\tau$ is the output of $\mbox{\textit{AdaGrad}}$, $M = C/b$, $\eta \geq \delta\|\tilde{\w} - \w^*\|$. Furthermore, if we choose $\delta, b, \eta$ such that $\frac{\delta^2b^2}{2\sigma^2C}\geq 1$ and $\eta = \frac{\delta^2b\|\tilde{\w} - \w^*\|}{\sqrt{2C}\sigma} \leq \frac{\alpha\delta^2}{8L}$, then we have
\begin{align}\label{equation:one stage adagrad}
	\EB[F(\w_\tau) - F(\w^*)] \leq \frac{4\sqrt{2}\|\tilde{\w} - \w^*\|\sigma}{\alpha\sqrt{C}}.
\end{align}

\end{lemma}

According to Lemma~\ref{lemma:adagrad}, we actually prove a $\mathcal{O}(1/\sqrt{C})$ error while \citep{DBLP:conf/colt/DuchiHS10} proves a $\mathcal{O}(1/\sqrt{M})$ error, where $C = Mb$. Furthermore, our one-stage training error is independent of the input $\delta$. Since the exact upper bound of $\|\nabla f(\w;\xi)\|_{\infty}$ is usually unknown, we can set a large $\delta$ without loss of training error in our presented AdaGrad. Hence, the error bound in~(\ref{equation:one stage adagrad}) is better than that in~\citep{DBLP:conf/colt/DuchiHS10} in which a large $\delta$ may lead to a large error.

Then we have the following convergence result for AdaSEBS:
\begin{theorem}\label{theorem:mb_adagrad}
	Let $\{\tilde{\w}_s\}$ be the sequence produced by Alorithm~\ref{alg:mb_adagrad}, $F(\tilde{\w}_1) - F(\w^*) \leq \epsilon_1$.
	By setting $C_s = \theta/\epsilon_s$, $S =\log_\rho(\epsilon_1/\epsilon)$ and
	\begin{align*}
		b_s = \frac{\alpha\sigma\sqrt{\mu\theta}}{8L\epsilon_s}, \eta_s = \frac{\alpha\delta^2}{8L},
	\end{align*}
	we obtain $\EB[F(\tilde{\w}_{s+1}) - F(\w_s)] \leq \epsilon$.
Here, $\theta = \frac{64\sigma^2\rho^2}{\alpha^2\mu}$, $\delta \geq \frac{8L\sqrt{2\epsilon_1}}{\alpha\sqrt{\mu}}$, $\epsilon_{s+1} = \epsilon_s/\rho, s\geq 1$.
\end{theorem}
Similar to SEBS for vanilla SGD, we also obtain $\mathcal{O}(\frac{\sigma^2}{\alpha^2\mu\epsilon})$ computation complexity and $\mathcal{O}(\frac{L}{\alpha^2\mu}\log(\frac{1}{\epsilon}))$ iteration complexity.

Recently, there is another work about stagewise AdaGrad, called SADAGrad~\citep{DBLP:conf/icml/ChenXCY18}, which mainly focuses on the case that $\g_m$ is sparse. SADAGrad adopts a constant batch size and decreases the learning rate geometrically by stage. Under convex and quadratic growth condition, SADAGrad achieves an iteration complexity of $\mathcal{O}(\frac{d\delta^2}{\mu\epsilon})$, which is dependent on dimension $d$. Hence, AdaSEBS is better than SADAGrad.

\section{Experiments}
First, we verify the theory about the relationship between batch size and model initialization. We consider a synthetic problem:
\begin{align}\label{equation:synthetic}
	\min_{\w\in \RB^{100}}F(\w) = \frac{1}{2n}\sum_{i=1}^n(\w - \xi_i)^T\D(\w - \xi_i),
\end{align}
where $n=10^4$, each data $\xi_i \in \RB^{100}$ is sampled from the gaussian distribution $\NM(\0,\I)$, and $\D$ is a diagonal matrix with $D(i,i) = i$. The corresponding $\alpha, \mu, L$ are $1, 1, 100$, respectively, and the optimal solution of (\ref{equation:synthetic}) is $\w^* = \sum_{i=1}^n\xi_i/n$. We run vanilla SGD in~(\ref{equation:update of vanilla sgd}) with a fixed computation complexity $C = n$ to solve~(\ref{equation:synthetic}). We set the model parameter initialization $\w_1 = \w^* + x*\d$, where $\|\d\| = 1$, $x\in[10,100]$. For each $x$, we aim to find the optimal batch size that can achieve the smallest value of $\|\hat{\w} - \w^*\|$, where $\hat{\w}$ is the output of vanilla SGD algorithm in~(\ref{equation:update of vanilla sgd}). We repeat 50 times and the average result about the optimal batch size is presented in Figure~\ref{figure:relation}. We can find that the optimal batch size is almost proportional to $1/\|\w_1 - \w^*\|$, and a larger learning rate implies a larger optimal batch size. These phenomenons are consistent with our theory in~(\ref{equation:basic relation}) where $\eta^*/b^* \propto \|\w_1 - \w^*\|$ for a fixed computation complexity $C$.
\begin{figure}[!thb]
\centering
\includegraphics[width=7cm]{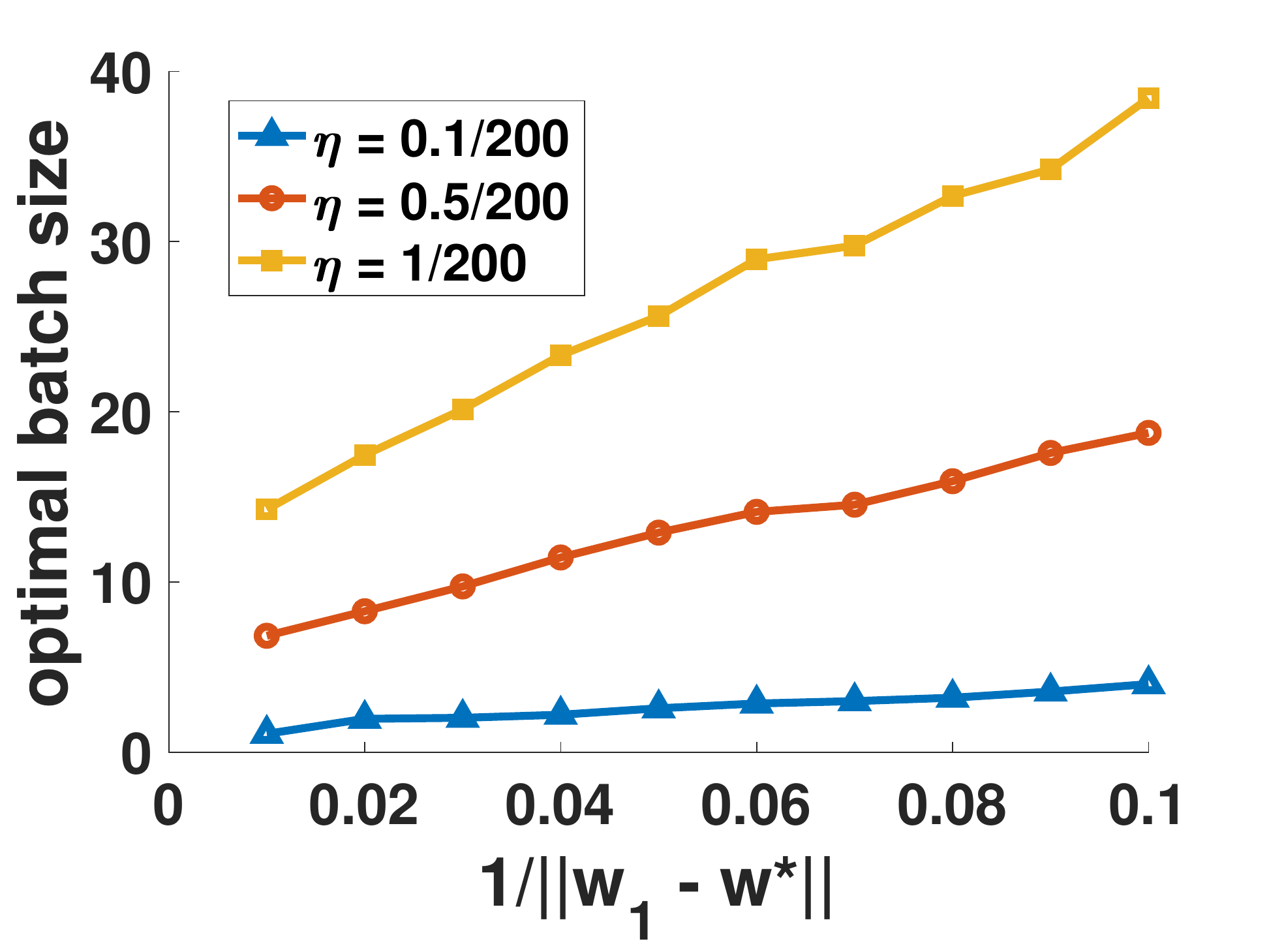}
\caption{Relationship between the optimal batch size and model initialization in vanilla SGD for solving problem~(\ref{equation:synthetic}).}\label{figure:relation}
\end{figure}

\begin{figure*}[t]
\centering
\subfigure{
	\includegraphics[width = 5.0cm]{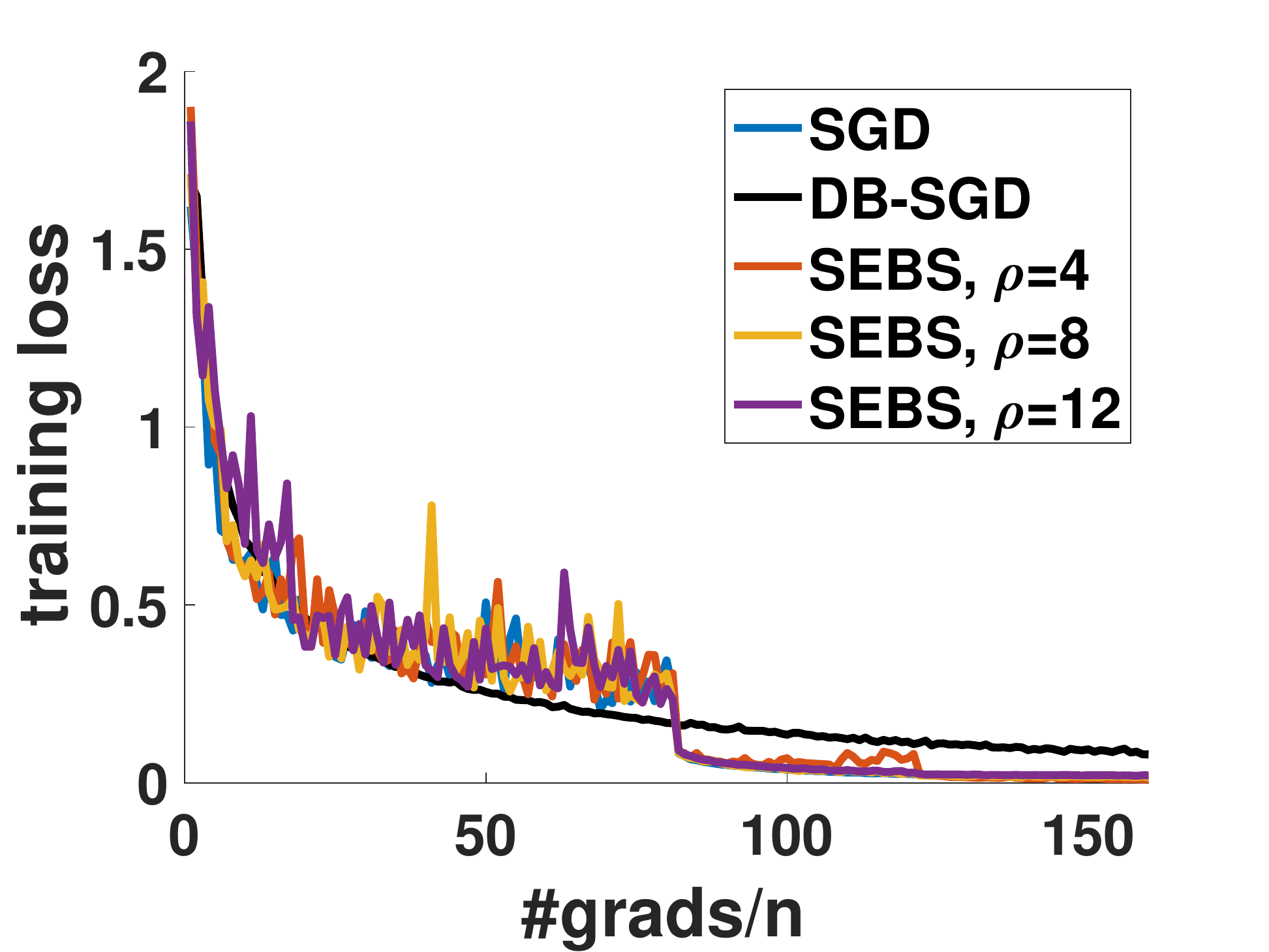}
	\includegraphics[width = 5.0cm]{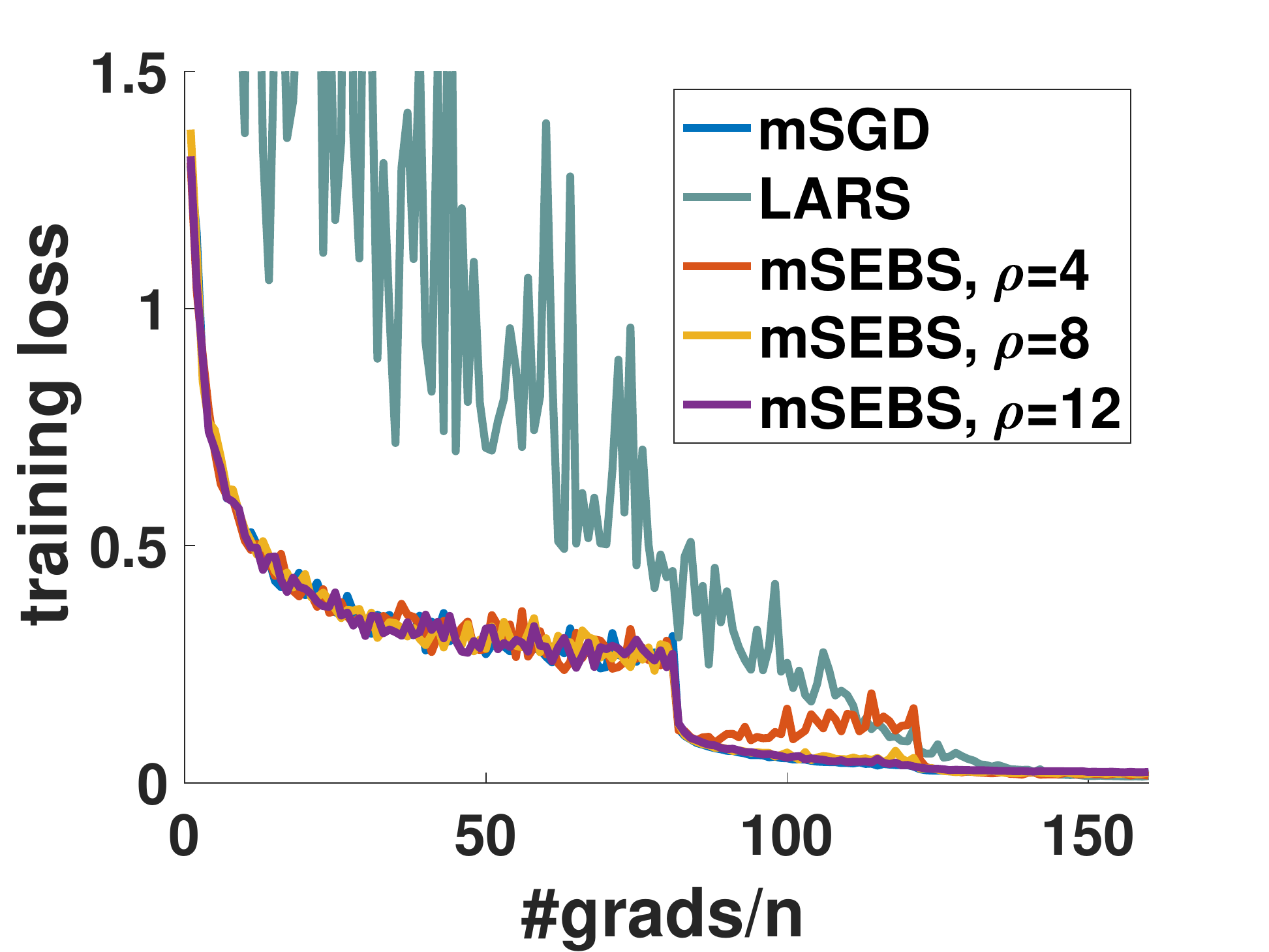}
	\includegraphics[width = 5.0cm]{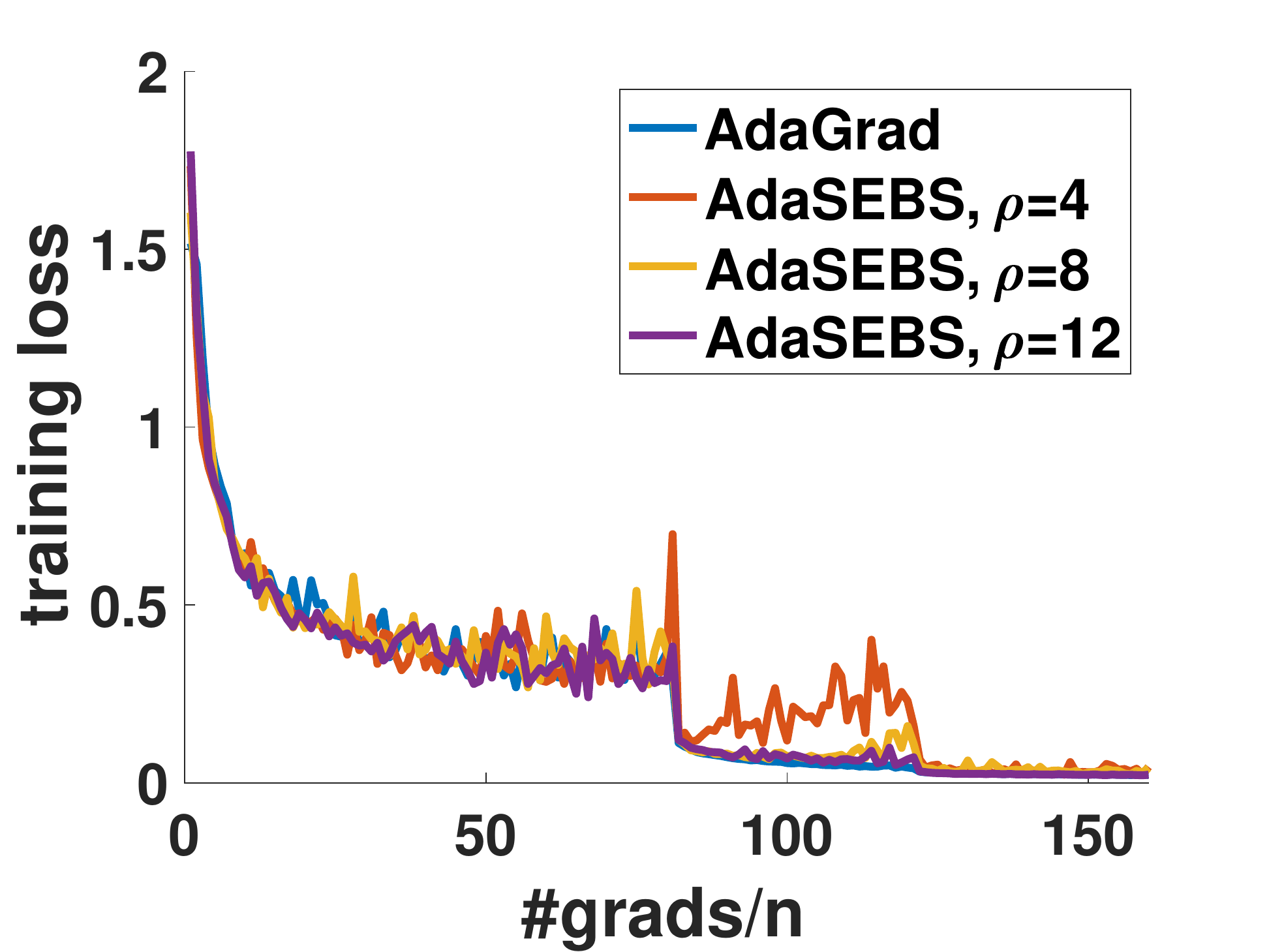}
}
\subfigure{
	\includegraphics[width = 5.0cm]{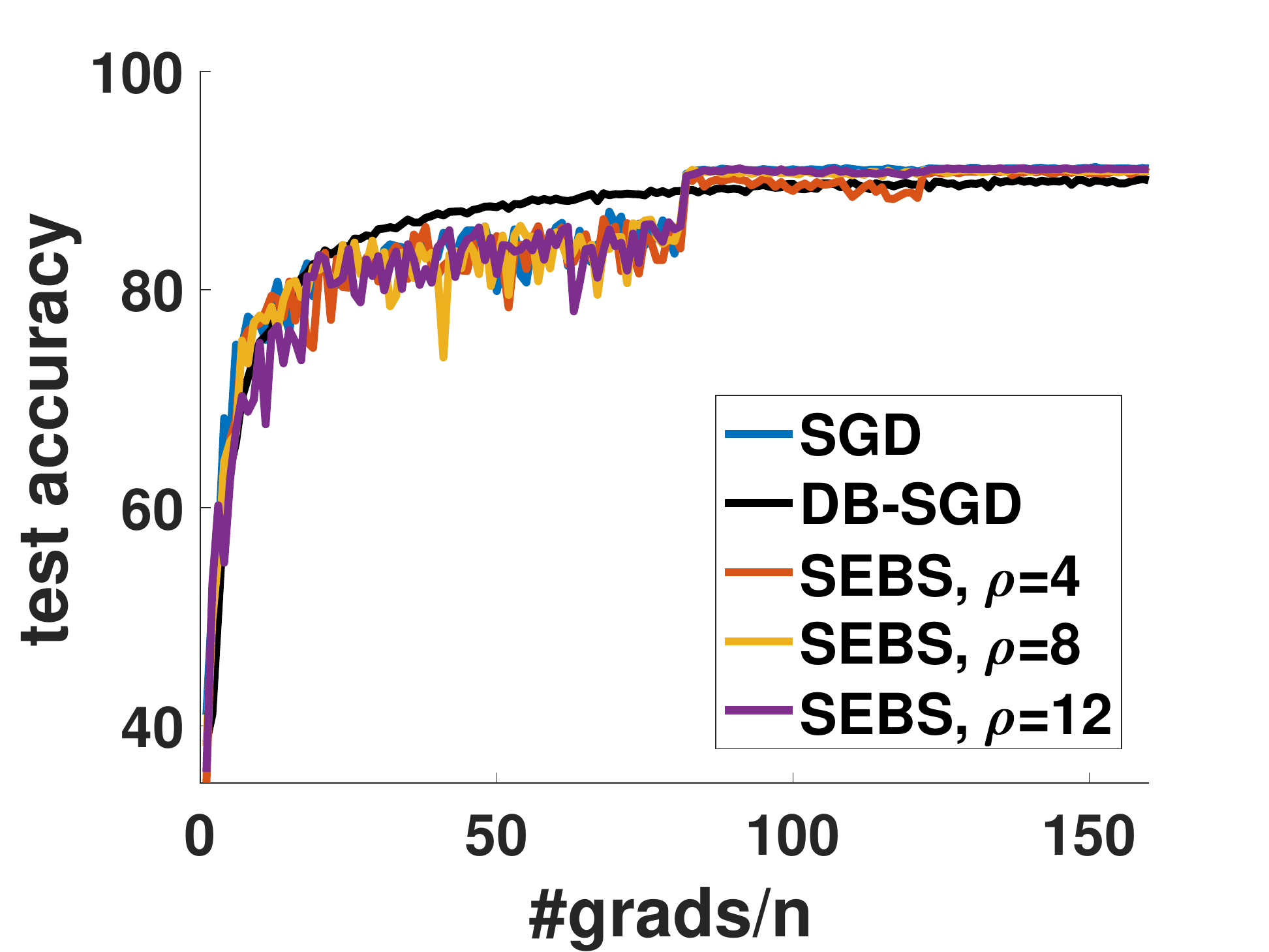}
	\includegraphics[width = 5.0cm]{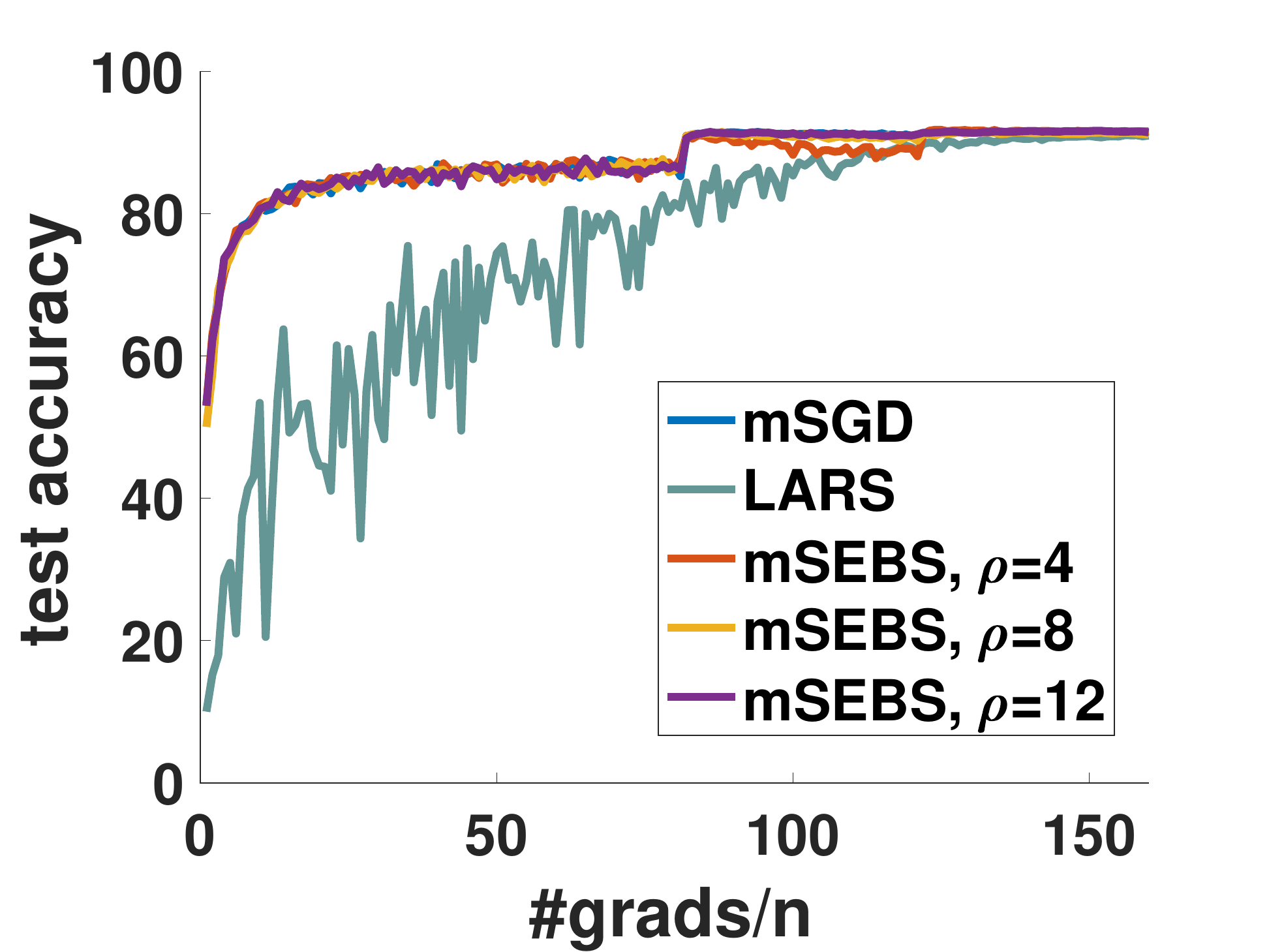}
	\includegraphics[width = 5.0cm]{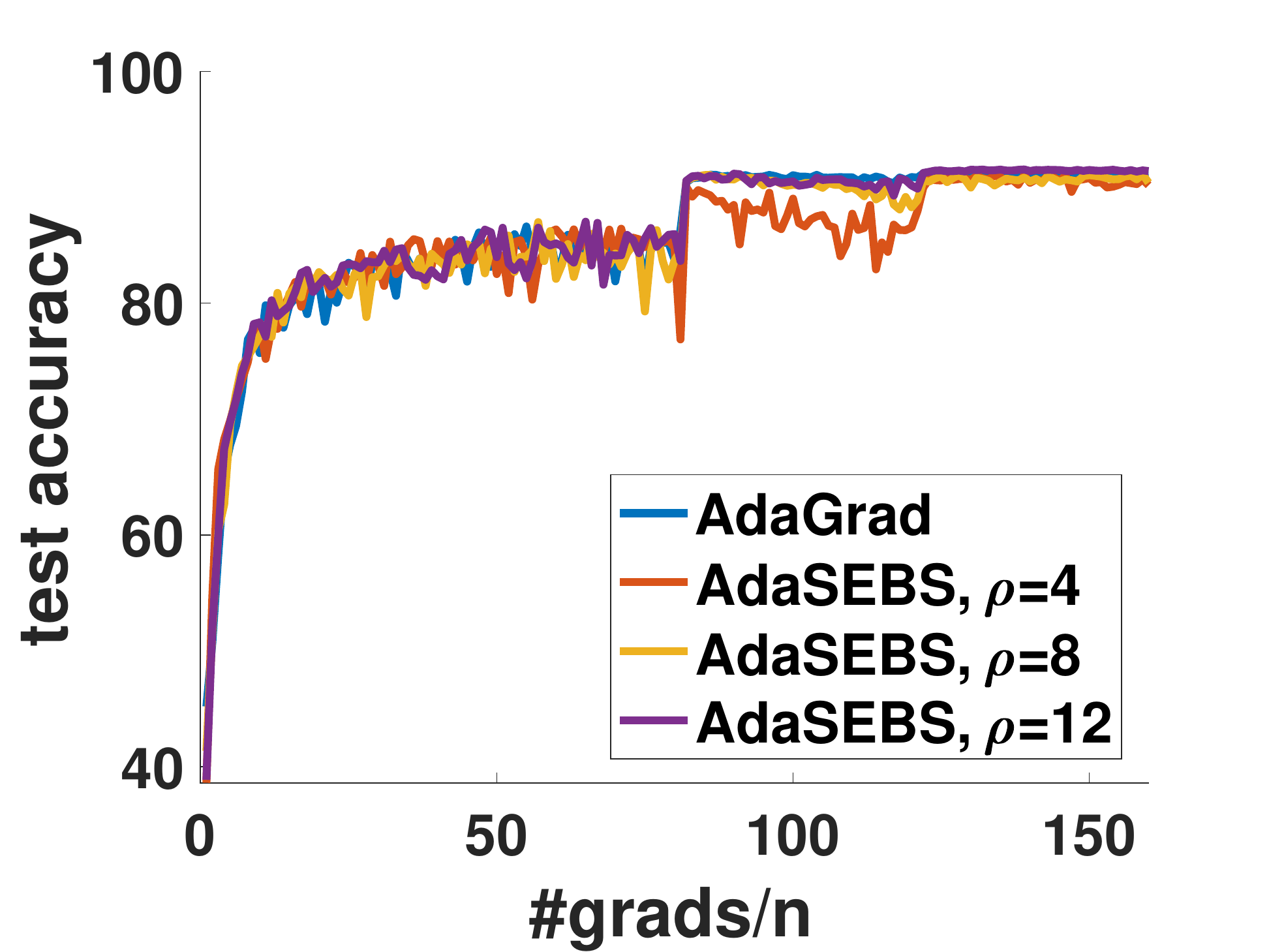}
}
\caption{Learning curves of training ResNet20 on CIFAR10 with different methods. The learning rates for SGD, mSGD and AdaGrad experiments are 0.5, 0.1, 1, respectively. The $\gamma$ in SGD experiment is $10^4$ and the $\delta$ in AdaGrad experiments is $1$.}\label{figure:res20}
\end{figure*}

Next, we consider a real problem which trains ResNet20 with 0.0001 weight decay on CIFAR10. The experiments are conducted on the PyTorch platform with an NVIDIA V100 GPU~(32G GPU memory). For classical stagewise methods, we follow~\citep{DBLP:conf/cvpr/HeZRS16} which divides the learning rate by 10 at the $80$, $120$ epochs. According to our theory that $\eta_s/b_s = \mathcal{O}(\epsilon_s)$, in SEBS, mSEBS and AdaSEBS, the learning rate is constant and the batch size is scaled by $\rho$ at the $80$, $120$ epochs. We set $\rho = 4, 8, 12$ for illustration. In the experiments about vanilla SGD, we also compare SEBS with DB-SGD~\citep{DBLP:conf/icml/YuJ19} in which the scaling ratio for batch size is $1.02$. The initial batch size of these methods is $128$. In the experiments about momentum SGD, we also compare mSEBS with the large batch training method LARS~\citep{DBLP:journals/corr/abs-1708-03888}. The poly power and warm-up of LARS are the same as that in ~\citep{DBLP:journals/corr/abs-1708-03888}. We set the batch size, based learning rate, scaling factor of LARS as 4096, 3.2, 0.01. The results are presented in Figure~\ref{figure:res20}. We can find that SEBS, mSEBS and AdaSEBS can achieve similar performance, measured based on \emph{computation complexity}~(epochs), as classical stagewise counterparts respectively, especially when $\rho$ is large. When measured based on \emph{iteration complexity} which is directly related to computation time or wall-clock time, SEBS, mSEBS and AdaSEBS are more efficient than their classical stagewise counterparts respectively. In particular, classical stagewise counterparts expend $62.5$k parameter updates, while SEBS with $\rho = 12$ only expends $32.6$k parameter updates. Since DB-SGD increases the batch size in every epoch, it falls into a local minimum and the accuracy is worse than SEBS. We also try some other scaling ratios for DB-SGD and DB-SGD still cannot achieve performance as good as classical stagewise SGD and SEBS on either training loss or test accuracy. Different from \mbox{DB-SGD}, SEBS increases the batch size after a stage which contains several epochs, and hence it achieves better performance than DB-SGD. Although LARS expends fewer parameter updates than mSEBS, it only achieves test accuracy of $90.97\%$, while mSGD and mSEBS with $\rho=12$ achieve test accuracy of $91.74\%$. We also try to set the scaling factor of LARS as that in~\citep{DBLP:journals/corr/abs-1708-03888}, but the test accuracy further drops $5\%$.

\begin{table*}[t]
\centering
\begin{tabular}{l|lcccc}
  ~        &  ~    & initial $b$ & initial $\eta$ & \#parameter updates & test accuracy \\ \hline
  \multirow{3}*{ResNet18} & mSGD  & 256 & 0.1 & 450k            & 69.56\%       \\
                          & mSGD* & 256 & 0.1 & 450k            & 69.90\%       \\
                          & mSEBS & 256 & 0.1 & 160k            & 69.75\% \\ \hline
  \multirow{4}*{ResNet50} & mSGD  & 256 & 0.1 & 450k            & 75.85\% \\
                          & mSGD* & 256 & 0.1 & 450k            & 75.85\% \\
                          & mSEBS & 256 & 0.1 & 160k            & 75.87\% \\
                          & \citep{DBLP:conf/iclr/SmithKYL18} & 8192 & 3.2 & 5.63k & 73.44\% \\ \hline
\end{tabular}
\caption{Empirical results on ImageNet. mSGD* is the momentum SGD implemented on PyTorch, in which the momentum is not reset to zero at the 30, 60 epochs.}\label{table:imagenet}
\vspace{-0.2cm}
\end{table*}

We also compare mSEBS with momentum SGD~(mSGD) by training ResNet18 and ResNet50 with 0.0001 weight decay on ImageNet. Data augmentation and initialization of $\w$~(including the parameters of batch normalization layers) follow the code of PyTorch~\footnote{https://github.com/pytorch/examples/tree/master/imagenet.}. In mSGD and mSEBS, the initial batch size is $256$ and the learning rate is $0.1$. Following~\citep{DBLP:conf/cvpr/HeZRS16}, we divide the learning rates of mSGD by $10$ and scale the batch size of mSEBS by $12$, at the $30$, $60$ epochs. The results are presented in Table~\ref{table:imagenet}. We can see that mSEBS achieves the same performance as momentum SGD on test accuracy. mSEBS scales the batch size to 36k after 60 epochs and saves about $64\%$ parameter updates in total. We also run the large batch training method in~\citep{DBLP:conf/iclr/SmithKYL18} to train ResNet50: the initial batch size and learning rate are 8192 and 3.2 respectively, the batch size is scaled by 10 at the 30 epoch, the learning rate is divided by 10 at the 60, 80 epochs. Although the method in~\citep{DBLP:conf/iclr/SmithKYL18} expends fewer parameter updates, its accuracy drops $2.4\%$. Hence, SEBS is better than classical stagewise methods and more universal than large batch training methods.

\section{Conclusion}
In this paper, we propose a novel method called SEBS to set proper batch size for SGD-based machine learning. Both theoretical and empirical results show that SEBS can reduce the number of parameter updates without loss of training error and test accuracy, compared to classical stagewise \mbox{SGD} methods.


\vskip 0.2in
\bibliography{ref}

\newpage
\appendix
\onecolumn

\section{SEBS}\label{appendix:sebs}

\subsection{Proof of Lemma~\ref{lemma:one stage error for pSGD}}
According to the updates $\w_{m+1} = \mathop{\arg\min}_{\w}~\g_m^T\w + \frac{1}{2\eta}\|\w - \w_m\|^2 + r(\w)$, we get that
\begin{align*}
	(\g_m + \frac{1}{\eta}(\w_{m+1} - \w_m) + \nabla r(\w_{m+1}))^T(\w_{m+1} - \w)\leq 0, \forall \w.
\end{align*}
Using the fact that $ab = \frac{1}{2}(a^2+b^2-(a-b)^2)$ and $r(\w)$ is convex, we obtain
\begin{align*}
	 & \g_m^T(\w_{m} - \w) + r(\w_{m+1}) - r(\w) \\
\leq & \frac{\|\w_m - \w\|^2}{2\eta} - \frac{\|\w_{m+1} - \w\|^2}{2\eta} + \g_m^T(\w_{m} - \w_{m+1}) - \frac{1}{2\eta}\|\w_{m} - \w_{m+1}\|^2 \\
\leq & \frac{\|\w_m - \w\|^2}{2\eta} - \frac{\|\w_{m+1} - \w\|^2}{2\eta} + \frac{\eta}{2}\|\g_m\|^2.
\end{align*}
Taking expectation on both sides, we obtain
\begin{align*}
     & \EB[\nabla F(\w_m)^T(\w_m - \w)  + r(\w_{m+1}) - r(\w)] \\
\leq & \EB[\frac{\|\w_m - \w\|^2}{2\eta} - \frac{\|\w_{m+1} - \w\|^2}{2\eta}] + \frac{\eta}{2}\EB[\|\g_m - \nabla F(\w_m)\|^2 + \| \nabla F(\w_m)\|^2] \\
\leq & \EB[\frac{\|\w_m - \w\|^2}{2\eta} - \frac{\|\w_{m+1} - \w\|^2}{2\eta}] + \frac{\eta}{2}\EB[\frac{\sigma^2}{b} + \| \nabla F(\w_m)\|^2].
\end{align*}
Summing up from $m=1$ to $M$, we obtain
\begin{align*}
     & \sum_{m=1}^{M}\EB[(\alpha - L\eta)(F(\w_m) - F(\w)) + r(\w_{m}) - r(\w)] \\
\leq & \frac{\|\w_1 - \w\|^2}{2\eta} + \frac{M\sigma^2\eta}{2b} + \EB[r(\w_1) - r(\w_{M+1})] \\
\leq & \frac{\|\w_1 - \w\|^2}{2\eta} + \frac{M\sigma^2\eta}{2b},
\end{align*}
which implies
\begin{align*}
	(\alpha - L\eta)\EB[F(\w_{\tau}) - F(\w^*)] \leq \frac{\|\tilde{\w} - \w_*\|^2}{2M\eta} + \frac{\sigma^2\eta}{2b} + \frac{1}{2\gamma}\|\tilde{\w} - \w_*\|^2.
\end{align*}
Since $\eta \leq \frac{\alpha}{2L}$, we obtain
\begin{align*}
	\EB[F(\w_{\tau}) - F(\w_*)] \leq \frac{\|\tilde{\w} - \w_*\|^2}{\alpha M\eta} + \frac{\sigma^2\eta}{\alpha b} + \frac{1}{\alpha\gamma}\|\tilde{\w} - \w_*\|^2.
\end{align*}

\subsection{Proof of Theorem~\ref{theorem:mb_sgd}}
Since $F(\tilde{\w}_1) - F(\w^*) \leq \epsilon_1$, we use the induction to prove the result. Assuming $F(\tilde{\w}_s) - F(\w^*) \leq \epsilon_s$, and using the PL condition, we obtain
\begin{align*}
	 \EB[F(\tilde{\w}_{s+1}) - F(\w^*)] \leq & (\frac{2}{\mu\alpha M_s\eta_s} + \frac{2}{\mu\alpha\gamma})(F(\tilde{\w}_s) - F(\w^*)) + \frac{\sigma^2\eta_s}{\alpha b_s} \\
\leq & (\frac{2b_s}{\mu\alpha C_s\eta_s} + \frac{2}{\mu\alpha\gamma})\epsilon_s + \frac{\sigma^2\eta_s}{\alpha b_s}.
\end{align*}
Since $C_s = \frac{\theta}{\epsilon_s}$ and $\eta_s = \sqrt{\frac{2b_s^2\epsilon_s}{\mu\sigma^2C_s}} = \frac{\sqrt{2}b_s\epsilon_s}{\sigma\sqrt{\mu\theta}} \leq \frac{\alpha}{2L}$, we obtain
\begin{align*}
	\EB[F(\tilde{\w}_{s+1}) - F(\w^*)] \leq \frac{2\sqrt{2}\sigma\epsilon_s}{\alpha \sqrt{\mu\theta}} + \frac{2\epsilon_s}{\mu\alpha\gamma}.
\end{align*}
By setting $\frac{1}{\gamma} \leq \frac{\mu\alpha}{4\rho}$ and $\theta\geq \frac{32\sigma^2\rho^2}{\alpha^2\mu}$, we obtain
\begin{align*}
	\EB[F(\w_{s+1}) - F(\w_*)] \leq \frac{\epsilon_s}{\rho} = \epsilon_{s+1}.
\end{align*}
Finally, we obtain that when $S = \log(\epsilon_1/\epsilon)$, $\EB[F(\w_{S+1}) - F(\w_*)] \leq \epsilon$.

\subsection{Proof of Lemma~\ref{lemma:delta_m}}
If $\mathcal{B}_{1,m} = \mathcal{B}_{2,m} \triangleq \mathcal{B}_{m}$, then we have
\begin{align*}
 	\delta_{m+1} \leq & \|\frac{\gamma\w_{1,m} + \eta\w_{1,1} - \gamma\eta\nabla f_{\mathcal{B}_m}(\w_{1,m})}{\gamma + \eta} - \frac{\gamma\w_{2,m} + \eta\w_{2,1} - \gamma\eta\nabla f_{\mathcal{B}_m}(\w_{2,m})}{\gamma + \eta}\| \\
	\leq & \frac{\eta}{\gamma + \eta}\|\w_{1,1}-\w_{2,1}\| + \frac{\gamma}{\gamma + \eta}\|\w_{1,m} - \eta\nabla f_{\mathcal{B}_m}(\w_{1,m}) - (\w_{2,m} - \eta\nabla f_{\mathcal{B}_m}(\w_{2,m}))\| \\
	=    & \frac{\eta}{\gamma + \eta}\delta_1 + \frac{\gamma(1+L\eta)}{\gamma + \eta}\delta_m.
\end{align*}
If $\mathcal{B}_{1,m} \neq \mathcal{B}_{2,m}$, then we have
\begin{align*}
	\delta_{m+1} \leq & \|\frac{\gamma\w_{1,m} + \eta\w_{1,1} - \gamma\eta\nabla f_{\mathcal{B}_{1,m}}(\w_{1,m})}{\gamma + \eta} - \frac{\gamma\w_{2,m} + \eta\w_{2,1} - \gamma\eta\nabla f_{\mathcal{B}_{2,m}}(\w_{2,m})}{\gamma + \eta}\| \\
	\leq & \frac{\eta}{\gamma + \eta}\delta_1 + \frac{\gamma}{\gamma + \eta}\delta_m + \frac{\gamma\eta}{\gamma + \eta}\|\nabla f_{\mathcal{B}_{1,m}}(\w_{1,m}) - \nabla f_{\mathcal{B}_{2,m}}(\w_{2,m})\| \\
	\leq & \frac{\eta}{\gamma + \eta}\delta_1 + \frac{\gamma}{\gamma + \eta}\delta_m + \frac{(b-1)L\gamma\eta}{b(\gamma + \eta)}\delta_m + \frac{2\gamma\eta G}{b(\gamma + \eta)},
\end{align*}
where the last inequality uses the fact $\mathcal{B}_{1,m}$ and $  \mathcal{B}_{2,m}$ differ in at most one instance.

\subsection{Proof of Theorem~\ref{theorem:stability}}
Let $\mathcal{E}$ be the event that SEBS picking $\xi_{i_0}$ for the first time happens at the $m_0$-iteration of the last stage, then we have
\begin{align*}
    \EB|f(\tilde{\w}_{1};\xi) - f(\tilde{\w}_{2};\xi)| \leq & G\PB(\mathcal{E})\EB[\|\tilde{\w}_{1} - \tilde{\w}_{2}\||\mathcal{E}] + \PB(\mathcal{E}^c)\EB[|f(\tilde{\w}_{1};\xi) - f(\tilde{\w}_{2};\xi)|\mathcal{E}^c] \\
\leq & G\EB[\|\tilde{\w}_{1} - \tilde{\w}_{2}\||\mathcal{E}] + \PB(\mathcal{E}).
\end{align*}
Let $j$ be the total iterations that mb-SGD using $\xi_{i_0}$ for the first time. Then we obtain
\begin{align*}
	\PB(\mathcal{E}) \leq & \PB(j \leq m_0 + \sum_{s=1}^{S-1}M_s) \leq \frac{bm_0}{n} + \sum_{s=1}^{S-1}\frac{b_sM_s}{n} \leq \frac{C}{n} + \frac{bm_0}{n},
\end{align*}
in which we use the inequality $\PB(j\leq H) \leq \sum_{h=1}^H \PB(j=h)$. Using Lemma~\ref{lemma:delta_m}, we obtain
\begin{align*}
	\EB[\delta_{m+1}|\mathcal{E}] \leq & (1-\frac{b}{n})\EB[\frac{\eta}{\gamma+\eta}\delta_1 + \frac{\gamma(1+L\eta)}{\gamma+\eta}\delta_m|\mathcal{E}] + \frac{b}{n}\EB[\frac{\eta}{\gamma+\eta}\delta_1 + \frac{b\gamma + (b-1)L\gamma\eta}{b(\gamma+\eta)}\delta_m + \frac{2\gamma\eta G}{b(\gamma+\eta)}|\mathcal{E}] \\
	= & (\frac{\gamma(1+L\eta)}{\gamma+\eta} - \frac{\gamma L\eta}{n(\gamma+\eta)})\EB[\delta_m|\mathcal{E}] + \frac{2\gamma\eta G}{n(\gamma+\eta)} \\
	\leq & (1 + (1 - \frac{\gamma}{n(\gamma+\eta)})L\eta)\EB[\delta_m|\mathcal{E}] + \frac{2\gamma\eta G}{n(\gamma+\eta)} \\
	\leq & (1 + \frac{2L}{\mu\alpha m})\EB[\delta_m|\mathcal{E}] + \frac{4\gamma G}{\mu\alpha(\gamma+\eta)nm},
\end{align*}
which implies that
\begin{align*}
	\EB[\delta_{m}|\mathcal{E}] \leq \frac{\gamma}{\gamma+\eta}\frac{2G}{ Ln}(\frac{M}{m_0})^{\frac{2L}{\mu\alpha}}.
\end{align*}
Then we obtain
\begin{align*}
	& \EB|f(\tilde{\w}_{1};\xi) - f(\tilde{\w}_{2};\xi)| \leq \frac{\gamma}{\gamma+\eta}\frac{2G^2}{Ln}(\frac{M}{m_0})^{\frac{2L}{\mu\alpha}} + \frac{bm_0}{n} +\frac{\sum_{t=1}^{T-1}C_t}{n}.
\end{align*}
By setting $m_0 = (\frac{4\gamma G^2}{(\gamma+\eta)b\mu\alpha})^{\frac{1}{2L/(\mu\alpha) + 1}}M^{\frac{2L/(\mu\alpha)}{2L/(\mu\alpha) + 1}}$, $q = 2L/(\mu\alpha)$, we obtain
\begin{align*}
	\EB|f(\tilde{\w}_{1};\xi) - f(\tilde{\w}_{2};\xi)| \leq & \frac{C}{n} + \frac{(1 + 1/q)}{n}(\frac{4\gamma G^2}{(\gamma+\eta)\mu\alpha})^{\frac{1}{1+q}}(bM)^{\frac{q}{q+1}} \\
=    & \frac{C}{n} + \frac{(1 + 1/q)}{n}(\frac{4\gamma G^2}{(\gamma+\eta)\mu\alpha})^{\frac{1}{1+q}}C^{\frac{q}{q+1}}
\end{align*}

\newpage
\section{Momentum SEBS}\label{appendix:msebs}

Similar to the convergence analysis for SEBS, we first establish the one-stage training error for for mSEBS:
\begin{lemma}\label{lemma:M}
	(One-stage training error for for mSEBS)~Let $\{\w_m\}$ be the sequence produced by $\mbox{\textit{mSGD}}(f,\mathcal{I}, \beta, \tilde{\w}, \eta, b, C)$. Then we have
	\begin{align}\label{equation:lemma M}
		& (\alpha - \frac{(1+\beta)L\eta}{(1-\beta)^2})\EB [F(\w_\tau) - F(\w^*)] \nonumber \\
	    \leq & \frac{(1-\beta)\|\tilde{\w} - \w^*\|^2}{2M\eta} + \frac{\sigma^2\eta}{2b(1-\beta)^2},
\end{align}
where $\w_\tau$ is the output of $\mbox{MSGD}$ and $M = C/b$.
\end{lemma}
Then we have the following convergence result:
\begin{theorem}\label{theorem:mb_msgd}
	Let $F(\tilde{\w}_1) - F(\w^*) \leq \epsilon_1$ and $\{\tilde{\w}_s\}$ be the sequence produced by
	$$\tilde{\w}_{s+1} = \mbox{\textit{mSGD}}(f, \mathcal{I}, \beta, \tilde{\w}_s, \eta, b_s, C_s),$$ where $C_s = \theta/\epsilon_s$, and
	\begin{align}\label{equation:key relation of msgd}
		\eta = \frac{\alpha(1-\beta)^2}{2(1+\beta)L}, b_s=\frac{\alpha\sigma\sqrt{\mu\theta(1-\beta)}}{2\sqrt{2}(1+\beta)L\epsilon_s}.
	\end{align}
	Then we obtain $\EB[F(\tilde{\w}_{s}) - F(\w^*)] \leq \epsilon_s, \forall s\geq 1$. If $S =\log_\rho(\epsilon_1/\epsilon)$, then $\EB[F(\tilde{\w}_{S+1}) - F(\w^*)] \leq \epsilon$. Here, $\theta = 8\sigma^2\rho^2/(\alpha^2\mu(1-\beta))$ and $\epsilon_{s+1} = \epsilon_s/\rho, s\geq 1$.
\end{theorem}

To prove the above lemma and theorem, we give the following lemma:
\begin{lemma}\label{lemma:momentum sgd v}
	Let $\{\w_m\}$ be the sequence produced by $\mathcal{M}(f,\mathcal{I}, \beta, \tilde{\w}, \eta, b, C)$ and define $\v_1 = \w_1$,
	\begin{align*}
		\v_m = \w_m + \frac{\beta}{1-\beta}(\w_m - \w_{m-1}), m\geq 2.
	\end{align*}
	Then we have
	\begin{align*}
		\|\v_{m+1} - \w^*\| \leq \|\v_m -\w^* - \frac{\eta}{1-\beta}\g_m\|.
	\end{align*}
	Specifically, if $\Omega = \RB^d$~\citep{ijcai2018-410}, then we have $\v_{m+1} = \v_m - \frac{\eta}{1-\beta}\g_m$.
\end{lemma}
\begin{proof}
Since $\v_m = \w_m + \frac{\beta}{1-\beta}(\w_m - \w_{m-1}), m\geq 2$, we obtain
\begin{align*}
	\|\v_{m+1} - \w^*\| = & \|\w_{m+1} + \frac{\beta}{1-\beta}(\w_{m+1} - \w_m) - \w^*\| \\
	= & \frac{1}{1-\beta}\|\w_{m+1} - (\beta\w_m + (1-\beta)\w^*)\| \\
	= & \frac{1}{1-\beta}\|\Pi_{\Omega}(\w_m -\eta\g_m + \beta(\w_m - \w_{m-1})) - (\beta\w_m + (1-\beta)\w^*)\| \\
	\leq & \frac{1}{1-\beta}\|\w_m -\eta\g_m + \beta(\w_m - \w_{m-1}) - (\beta\w_m + (1-\beta)\w^*)\| \\
	= & \frac{1}{1-\beta}\|(1-\beta)\w_m + \beta(\w_m - \w_{m-1}) -  (1-\beta)\w^* - \eta\g_m\| \\
	= & \|\v_m -\w^* - \frac{\eta}{1-\beta}\g_m\|.
\end{align*}
\end{proof}

\subsection{Proof of Lemma~\ref{lemma:M}}
Let $\u_{1} = \0, \u_m = \w_m - \w_{m-1}$. Then we have
\begin{align*}
    \EB\|\v_{m+1} - \w^*\|^2 = & \|\v_{m} - \w^*\|^2 - \frac{2\eta}{1-\beta}\nabla F(\w_m)^T(\v_{m} - \w^*) + \frac{\eta^2}{(1-\beta)^2}\EB\|\g_m\|^2 \\
=    & \|\v_{m} - \w^*\|^2 - \frac{2\eta}{1-\beta}\nabla F(\w_m)^T(\w_{m} - \w^*)  \\
     & + \frac{2\eta}{1-\beta}\nabla F(\w_m)^T(\w_m - \v_m) + \frac{\eta^2}{(1-\beta)^2}[\frac{\sigma^2}{b} + \|\nabla F(\w_m)\|^2] \\
=    & \|\v_{m} - \w^*\|^2 - \frac{2\eta}{1-\beta}\nabla F(\w_m)^T(\w_{m} - \w^*) - \frac{2\beta\eta}{(1-\beta)^2}\nabla F(\w_m)^T\u_m \\
     & + \frac{\eta^2}{(1-\beta)^2}[\frac{\sigma^2}{b} + \|\nabla F(\w_m)\|^2] \\
\leq & \|\v_{m} - \w^*\|^2 - \frac{2\alpha\eta}{1-\beta}(F(\w_m) - F(\w^*)) + \frac{2\beta\eta}{(1-\beta)^2}|\nabla F(\w_m)^T\u_m| \\
     & + \frac{\eta^2}{(1-\beta)^2}[\frac{\sigma^2}{b} + 2L(F(\w_m) - F(\w^*))]
\end{align*}
Using the Young's inequality that $(a+b)^2 \leq (1+\rho)a^2 + (1+\frac{1}{\rho})b^2, \forall \gamma>0, a, b$ and the fact $\|\u_{m+1}\| \leq \|\beta\u_m - \eta\g_m\|$, we obtain
\begin{align*}
	\EB\|\u_{m}\|^2 \leq & \beta\|\u_{m-1}\|^2 + \frac{\eta^2}{1-\beta}\EB\|\g_{m-1}\|^2 \leq \frac{\eta^2}{1-\beta}\sum_{i=1}^{m-1} \beta^{{m-1}-i}\|\g_i\|^2
\end{align*}
We denote $F_m = F(\w_m) - F(\w^*)$ for short. Then we have
\begin{align*}
\sum_{m=1}^M \frac{2\beta\eta}{(1-\beta)^2}\EB|\nabla F(\w_m)^T\u_m| \leq & \sum_{m=1}^M \frac{\beta\eta}{(1-\beta)^2}[\frac{\eta}{1-\beta}\EB\|\nabla F(\w_m)\|^2 + \frac{1-\beta}{\eta}\EB\|\u_m\|^2] \\
=    & \sum_{m=1}^M \{\frac{\beta\eta^2}{(1-\beta)^3}\EB\|\nabla F(\w_m)\|^2 + \frac{\beta}{(1-\beta)}[\frac{\eta^2}{1-\beta}\sum_{i=1}^{m-1} \beta^{{m-1}-i}\EB\|\g_i\|^2]\} \\
\leq & \sum_{m=1}^M \frac{2L\beta\eta^2}{(1-\beta)^3}\EB[F_m] + \frac{\beta\eta^2}{(1-\beta)^2}\sum_{m=1}^M\sum_{i=1}^{m-1} \beta^{{m-1}-i}\EB\|\g_i\|^2 \\
=    & \sum_{m=1}^M \frac{2L\beta\eta^2}{(1-\beta)^3}\EB[F_m] + \frac{\beta\eta^2}{(1-\beta)^2}\sum_{i=1}^{M-1}\EB\|\g_i\|^2\sum_{m=i+1}^{M} \beta^{{m-1}-i} \\
\leq & \sum_{m=1}^M \frac{2L\beta\eta^2}{(1-\beta)^3}\EB[F_m] + \frac{\beta\eta^2}{(1-\beta)^3}\sum_{i=1}^{M-1}\EB\|\g_i\|^2 \\
\leq & \frac{4L\beta\eta^2}{(1-\beta)^3}\sum_{m=1}^M \EB[F_m] + \frac{M\beta\sigma^2\eta^2}{b(1-\beta)^3},
\end{align*}
and hence
\begin{align*}
    & [\frac{2\alpha}{1-\beta} - \frac{4L\beta\eta}{(1-\beta)^3} - \frac{2L\eta}{(1-\beta)^2}]\EB [F(\w_\tau) - F(\w^*)] \leq \frac{\|\tilde{\w} - \w^*\|^2}{M\eta} + \frac{\beta\sigma^2\eta}{b(1-\beta)^3} + \frac{\sigma^2\eta}{b(1-\beta)^2}.
\end{align*}
i.e.,
\begin{align*}
	(\alpha - \frac{(1+\beta)L\eta}{(1-\beta)^2})\EB [F(\w_\tau) - F(\w^*)] \leq \frac{(1-\beta)\|\tilde{\w} - \w^*\|^2}{2M\eta} + \frac{\sigma^2\eta}{2b(1-\beta)^2}.
\end{align*}

\subsection{Proof of Theorem~\ref{alg:mb_msgd}}
Since $\eta_s \leq \alpha(1-\beta)^2/(2(1+\beta)L)$, using PL condition, we obtain
\begin{align*}
	\EB [F(\tilde{\w}_{s+1}) - F(\w^*)] \leq & \frac{2b_s(1-\beta)[F(\tilde{\w}_{s}) - F(\w^*)]}{\mu\alpha C_s\eta_s} + \frac{\sigma^2\eta_s}{\alpha b_s(1-\beta)^2}.
\end{align*}
Since $F(\tilde{\w}_1) - F(\w^*) \leq \epsilon_1$, we use the induction to prove the result. Assuming $F(\tilde{\w}_s) - F(\w^*) \leq \epsilon_s$, then we have
\begin{align*}
	\EB [F(\tilde{\w}_{s+1}) - F(\w^*)] \leq & \frac{2b_s(1-\beta)\epsilon_s}{\mu\alpha C_s\eta_s} + \frac{\sigma^2\eta_s}{\alpha b_s(1-\beta)^2} \\
	= & 2\sqrt{\frac{2\epsilon_s\sigma^2}{\mu C_s(1-\beta)\alpha^2}  } \tag{using definition of $\eta_s$} \\
	= & 2\sqrt{\frac{2\epsilon_s^2\sigma^2}{\mu \theta(1-\beta)\alpha^2}  } \tag{using definition of $C_s$} \\
	= & \frac{2\sqrt{2}\sigma\epsilon_s}{\sqrt{\mu \theta(1-\beta)}\alpha} \\
	= & \epsilon_s/\rho \tag{using definition of $\theta$} \\
	= & \epsilon_{s+1}.
\end{align*}

\section{AdaSEBS}\label{appendix:adasebs}

First, we have the following inequality: $\forall y\geq x > 0$,
\begin{align*}
\exists z\in [x,y], s.t. ~\ln(y) - \ln(x) = \frac{y - x}{z} \geq \frac{y - x}{y}.
\end{align*}

We define $\psi_m^*(\u) = \sup_\w(\w^T\u - \frac{1}{\eta}\psi_m(\w))$.
\subsection{Proof of Lemma~\ref{lemma:adagrad}}
Let $\z_0 = \0$, $\z_m = \sum_{i=1}^m \g_i$ . First, we have
\begin{align*}
	\alpha\sum_{m=1}^M [F(\w_m) - F(\w^*)] \leq & \sum_{m=1}^M \nabla  F(\w_m)^T(\w_m - \w^*) \\
	=    & \sum_{m=1}^M \g_m^T(\w_m - \w^*) + \sum_{m=1}^M \Delta_m \\
    =    & \sum_{m=1}^M \g_m^T\w_m  + \frac{1}{\eta}\psi_M(\w^*) + (-\sum_{m=1}^M \g_m^T\w^* - \frac{1}{\eta}\psi_M(\w^*)) + \sum_{m=1}^M \Delta_m \\
    \leq & \sum_{m=1}^M \g_m^T\w_m  + \frac{1}{\eta}\psi_M(\w^*) + \psi_M^*(-\sum_{m=1}^M\g_m) + \sum_{m=1}^M \Delta_m \\
    =    & \sum_{m=1}^M \g_m^T\w_m  + \frac{1}{\eta}\psi_M(\w^*) + \psi_M^*(-\z_M) + \sum_{m=1}^M \Delta_m,
\end{align*}
where $\Delta_m = (\nabla F(\w_m) - \g_m)^T(\w_m - \w^*)$. For the $\psi_M^*(-\z_M)$, we have
\begin{align*}
	\psi_M^*(-\z_M) = & -\w_{M+1}^T\z_M - \frac{1}{\eta}\psi_M(\w_{M+1}) \\
	\leq & -\w_{M+1}^T\z_M - \frac{1}{\eta}\psi_{M-1}(\w_{M+1}) \\
	\leq & \psi_{M-1}^*(-\z_M) \\
	\leq & \psi_{M-1}^*(-\z_{M-1}) - \g_M^T\nabla \psi_{M-1}^*(-\z_{M-1}) + \frac{\eta}{2}\|\g_M\|^2_{\psi_{M-1}^*} \\
	=    & \psi_{M-1}^*(-\z_{M-1}) - \g_M^T\w_{M} + \frac{\eta}{2}\|\g_M\|^2_{\psi_{M-1}^*}.
\end{align*}
Since $\psi_0^*(\0) \leq 0$, we obtain
\begin{align*}
	\sum_{m=1}^M \g_m^T\w_m + \psi_M^*(-\z_M) \leq & \sum_{m=1}^{M-1} \g_m^T\w_m + \psi_{M-1}^*(-\z_{M-1}) + \frac{\eta}{2}\|\g_M\|^2_{\psi_{M-1}^*} \\
	\leq & \cdots \\
	\leq & \frac{\eta}{2}\sum_{m=1}^M\|\g_m\|^2_{\psi_{m-1}^*}.
\end{align*}
Combining the above inequalities, we obtain
\begin{align*}
	\frac{\alpha}{M}\sum_{m=1}^M [F(\w_m) - F(\w^*)] \leq & \frac{1}{M\eta}\psi_M(\w^*) + \frac{\eta}{2M}\sum_{m=1}^M\|\g_m\|^2_{\psi_{m-1}^*} + \frac{1}{M}\sum_{m=1}^M \Delta_m.
\end{align*}
Since
\begin{align*}
	\sum_{m=1}^M\|\g_m\|^2_{\psi_{m-1}^*} = \sum_{m=1}^M\sum_{j=1}^d\frac{g_{m,j}^2}{\delta^2 + \sum_{i=1}^{m-1} g_{i,j}^2},
\end{align*}
we define $S_{0,j} = \delta^2, S_{m,j} = \sum_{i=0}^m g_{i,j}^2$, then
\begin{align*}
	\sum_{m=1}^M\|\g_m\|^2_{\psi_{m-1}^*} = & \sum_{j=1}^d\sum_{m=1}^M\frac{S_{m,j} - S_{m-1,j}}{S_{m-1,j}} \leq 2\sum_{j=1}^d\sum_{m=1}^M (\ln S_{m,j} - \ln S_{m-1,j}) \\ 
	= & 2\sum_{j=1}^d \ln(\frac{S_{M,j}}{S_{0,j}}) = 2\sum_{j=1}^d \ln(\frac{\sum_{m=1}^M g_{m,j}^2}{\delta^2} + 1).
\end{align*}
and
\begin{align*}
	\alpha X \leq \frac{1}{M\eta}\sum_{j=1}^d(\delta^2 + \|\g_{1:M,j}\|^2)(w_{1,j} - w_{j}^*)^2 + \frac{\eta}{M}\sum_{j=1}^d \ln(\frac{\|\g_{1:M,j}\|^2}{\delta^2}+1) + \frac{1}{M}\sum_{m=1}^M\Delta_m,
\end{align*}
where $X = \frac{1}{M}\sum_{m=1}^M [F(\w_m) - F(\w^*)]$. On the other hand, we have
\begin{align*}
	\EB[\sum_{j=1}^d\|\g_{1:M,j}\|^2] = \EB[\sum_{m=1}^M\|\g_m\|^2] \leq \EB[\frac{M\sigma^2}{b} + 2LMX].
\end{align*}

Using the inequality $\forall x\geq 0, \ln(x+1)\leq x$, we obtain
\begin{align*}
	\alpha \EB[X] \leq & \frac{\delta^2}{M\eta}\|\w_1 - \w^*\|^2 + (\frac{\|\w_1 - \w^*\|^2}{\eta} + \frac{\eta}{\delta^2})(\frac{1}{M}\sum_{m=1}^M\EB\|\g_m\|^2) \\
\leq & \frac{\delta^2}{M\eta}\|\w_1 - \w^*\|^2 + (\frac{\|\w_1 - \w^*\|^2}{\eta} + \frac{\eta}{\delta^2})(\EB[\frac{\sigma^2}{b} + 2LX]).
\end{align*}
Since $\eta^2 \geq \delta^2\|\w_1 - \w^*\|^2$ and $\eta \leq \alpha\delta^2/(8L)$, we obtain
\begin{align*}
    \frac{\alpha}{2} \EB[X] \leq & \frac{\delta^2}{M\eta}\|\w_1 - \w^*\|^2 + \frac{2\sigma^2\eta}{b\delta^2}.
\end{align*}
By choosing $\eta^2 = \frac{\delta^4\|\w_1 - \w^*\|^2b}{2M\sigma^2}$, we obtain
\begin{align*}
     \EB[X] \leq & \frac{4\sqrt{2}\|\w_1 - \w^*\|\sigma}{\alpha\sqrt{C}}.
\end{align*}
The conditions for the batch size and learning rate are
\begin{align*}
  & \delta\|\w_1 - \w^*\|\leq \eta = \frac{\delta^2\|\w_1 - \w^*\|b}{\sqrt{2C}\sigma} \leq \frac{\alpha\delta^2}{8L},
\end{align*}
i.e.,
\begin{align*}
  & \frac{\delta^2b^2}{2C\sigma^2} \geq 1 \mbox{~and~} \frac{\|\w_1 - \w^*\|b}{\sqrt{2C}\sigma} \leq \frac{\alpha}{8L}.
\end{align*}

\subsection{Proof of Theorem~\ref{theorem:mb_adagrad}}
Since $\delta \geq \frac{8\sqrt{2}L}{\sqrt{\mu}\alpha}$, $C_s = \theta/\epsilon_s$, $b_s = \frac{\sqrt{\mu\theta}\alpha\sigma}{8L\epsilon_s}$, $\eta_s = \frac{\alpha\delta^2}{8L}$, we have
\begin{align*}
	\frac{\delta^2b_s^2}{2C_s\sigma^2} = \frac{\delta^2}{2\sigma^2}\frac{\mu\theta\alpha^2\sigma^2}{64L^2\epsilon_s^2}\frac{\epsilon_s}{\theta} = \frac{\mu\alpha^2\delta^2}{128L^2\epsilon_s} \geq \frac{\mu\alpha^2\delta^2}{128L^2\epsilon_1} \geq 1.
\end{align*}
Then we can use Lemma~\ref{lemma:adagrad} and PL condition to obtain
\begin{align*}
	\EB[F(\tilde{\w}_{s+1}) - F(\w_s)] \leq & \frac{4\delta^2}{\mu\alpha M_s\eta_s}\EB[F(\tilde{\w}_s) - F(\w_s)] + \frac{4\sigma^2\eta_s}{\alpha b_s\delta^2}.
\end{align*}
Since $F(\tilde{\w}_1) - F(\w^*) \leq \epsilon_1$, we use the induction to proof the result. Assuming $\EB[F(\tilde{\w}_s) - F(\w^*)] \leq \epsilon_s$, then we have
\begin{align*}
	\EB[F(\tilde{\w}_{s+1}) - F(\w_s)] \leq & \frac{4b_s\delta^2\epsilon_s}{\mu\alpha C_s\eta_s} + \frac{4\sigma^2\eta_s}{\alpha b_s\delta^2}.
\end{align*}
Since $C_s = \theta/\epsilon_s$, $\eta_s = \delta^2\sqrt{\frac{b_s^2\epsilon_s}{\mu C_s\sigma^2}} = \frac{\delta^2 b_s\epsilon_s}{\sqrt{\mu\theta}\sigma}$, we obtain
\begin{align*}
	\EB[F(\tilde{\w}_{s+1}) - F(\w_s)] \leq & \frac{8\sigma}{\alpha\sqrt{\mu\theta}}\epsilon_s.
\end{align*}
Since $\theta = \frac{64\sigma^2\rho^2}{\mu\alpha^2}$, we obtain $\EB[F(\tilde{\w}_{s+1}) - F(\w^*)] \leq \epsilon_s/\rho = \epsilon_{s+1}$. Since $S = \log(\epsilon_1 / \epsilon)$, then we have $\EB[F(\tilde{\w}_{S+1}) - F(\w^*)] \leq \epsilon$.

\end{document}